\documentclass[11pt]{article}

\usepackage[final]{acl}

\usepackage{times}
\usepackage{latexsym}

\usepackage[T1]{fontenc}

\usepackage[utf8]{inputenc}

\usepackage{microtype}

\usepackage{inconsolata}

\usepackage{graphicx}
\usepackage{amsmath}
\usepackage{booktabs}
\usepackage{amsthm}
\usepackage[table]{xcolor}
\usepackage{algorithm}
\usepackage{algorithmic}
\usepackage{multirow}
\newtheorem{proposition}{Proposition}
\usepackage{amsfonts}

\title{SPaCe: Unlocking Sample-Efficient Large Language Models Training With Self-Pace Curriculum Learning}

\author{
 \textbf{Dai Do\textsuperscript{1}},
 \textbf{Manh Nguyen\textsuperscript{1}},
 \textbf{Svetha Venkatesh\textsuperscript{1}},
 \textbf{Hung Le\textsuperscript{1}}
\\
 \textsuperscript{1}Deakin Applied AI Initiative, Deakin University, Australia \quad
\\
{
\texttt{v.do@deakin.edu.au}
}
}

\begin{document}
\maketitle
\begin{abstract}
Large language models (LLMs) have shown strong reasoning capabilities when fine-tuned with reinforcement learning (RL). However, such methods require extensive data and compute, making them impractical under many realistic training budgets. Many existing pipelines sample training examples uniformly across steps or epochs, ignoring differences in difficulty, redundancy, and learning value, which slows learning and wastes computation. We propose \textbf{SPaCe}, a self-paced learning framework that enables efficient learning based on the capability of the model being trained through optimizing which data to use and when. First, we apply \emph{cluster-based data reduction} to partition training data by semantics and difficulty, extracting a compact yet diverse subset that reduces redundancy. Then, a \textit{multi-armed bandit} treats data clusters as arms, allocating training samples based on the model's solve rates and learning progress. Experiments across multiple reasoning benchmarks show that SPaCe achieves comparable or better accuracy than state-of-the-art baselines while using up to \(100\times\) fewer samples. Ablation studies and analyses further highlight the importance of both data clustering and adaptive selection. Our results demonstrate that carefully curated, performance-driven training curricula can unlock strong reasoning abilities in LLMs with minimal resources.
\end{abstract}

\section{Introduction}
Large Language Models (LLMs) have achieved remarkable progress in tasks requiring reasoning, problem-solving, and generalization, driven largely by scaling trends in model size, data, and compute \citep{google2024gemini, openai2024a}. As the cost and complexity of pretraining continue to rise, research attention has increasingly shifted toward post-training techniques, which aim to improve LLM capabilities more efficiently. Among these, Reinforcement Fine-Tuning (RFT) has emerged as a promising method that aligns model behavior with outcome-based reward signals, often relying on lightweight supervision without elaborate reward engineering or inference-time computation \citep{kumar2025llm, deepseekai2025deepseekr1incentivizingreasoningcapability, lightman2023lets}. 

Standard RFT uniformly samples batches from the full dataset \citep{deepseekai2025deepseekr1incentivizingreasoningcapability}. While simple, this approach ignores each example’s difficulty, informativeness, and uncertainty, wasting limited reward feedback on trivial or noisy instances and slowing convergence \citep{ouyang2022training, dong2023raft}. This raises two key underexplored dimensions: how to select which examples to train on, and how to present them to LLMs over time.

Data reduction methods prioritize informativeness by estimating example difficulty or uncertainty. For example, variance-based filtering based on multiple forward passes through a reference model \citep{wang2025reinforcementlearningreasoninglarge}. Although effective at denoising, these approaches incur significant computational overhead, impractical for resource-constrained models. They are also sensitive to the selected training examples and the uncertainty estimator, which hinders generalization to new LLMs.

In parallel, curriculum design plays a central role in guiding the learning trajectory \citep{10.1145/1553374.1553380}. As the model improves during fine-tuning, the useful difficulty level shifts dynamically, yet static curricula or random orders often fail to reflect this progression. Recent attempts at adaptivity filter examples with heuristic thresholds \citep{AdaRFT}, but such mechanisms are fragile, especially for small or weak models that rely on imperfect difficulty metrics, prematurely excluding challenging, informative examples, and stalling progress.

We introduce \textbf{SPaCe}, a self-paced RFT framework that improves training efficiency by selecting informative examples and adapting the training schedule online. We cast RFT as a Multi-Armed Bandit (MAB) problem \citep{sutton1998reinforcement}, where each arm is a cluster of examples with similar semantics and per-example attribute, enabling data selection beyond fixed heuristics. SPaCe first performs a one-time clustering by jointly clustering latent representations and per-example attributes, then reduces redundancy by retaining a fixed number of diverse representatives per cluster, selected by iteratively maximizing embedding distance \citep{wang2025reinforcementlearningreasoninglarge}. During training, SPaCe pulls an arm, sampling data from its cluster. The solve rate updates the bandit, while its negative (difficulty) guides Thompson Sampling, with additional downweighting for clusters whose hardness plateaus to encourage exploration \citep{thompson1933likelihood, russo2020tutorialthompsonsampling}. Overall, SPaCe prioritizes examples that are challenging yet learnable, avoiding repeated training on already-solved problems and remaining effective under tight budgets in low-resource RFT \citep{le2025reasoning}.

To evaluate our approach, we conduct extensive experiments on mathematical and logical reasoning tasks using various LLMs of different sizes. Results show that SPaCe significantly improves reasoning accuracy and robustness compared to both reinforcement learning and curriculum learning baselines. Notably, SPaCe also outperforms methods that rely on exhaustive search to select a single or a few training examples. Our analysis reveals how poorly designed curricula can get stuck in easy example regions, failing to leverage the diversity of the dataset. 

In summary, our contributions are threefold: (1) We propose \textbf{SPaCe}, a novel two-stage framework that reduces the number of training examples and optimizes the RFT progress using MAB. (2) Our method is lightweight, significantly reducing the number of training examples while adding minimal computational overhead, well-suited for low-resource settings. (3) Our extensive experiments demonstrate that SPaCe consistently outperforms existing curriculum and data reduction strategies.

\section{Related Work}
\subsection{Efficient Methods to Enhance Language Models.} 
The role of data in post-training LLMs remains an open research question. One line of work studies reinforcement fine-tuning (RFT). Several recent methods focus on curating high-quality mathematical datasets \citep{deepscaler2025, yu2025dapoopensourcellmreinforcement}, but they do not explicitly investigate which data is most effective for fine-tuning. More recently, alternative approaches have explored heuristic-based scoring methods, such as Learning Impact Measurement \citep{li2025limrrlscaling} and variance-based data selection \citep{wang2025reinforcementlearningreasoninglarge}. These methods improve data efficiency by enabling training on only a small subset of the data while still achieving strong reasoning performance. However, they typically require substantial precomputation, which limits their practicality in real-world settings, and they have not been thoroughly evaluated on small models with limited reasoning ability. Another line of work explores efficient post-training methods for LLMs, including alignment learning \citep{ji2024aligner, do-etal-2025-sample} and steering methods \citep{turner2025steering, do-etal-2025-dynamic}. In this paper, however, we focus specifically on methods related to SFT.

\subsection{Curriculum Learning for LLMs.}
Humans and animals learn more effectively when examples are presented in a meaningful order that gradually increases in complexity. Curriculum learning \citep{10.1145/1553374.1553380} and performance-guided training progression \citep{le2022episodic} have been applied to supervised and RL training. For LLMs, recent studies have explored how to organize training data to reduce computational cost and improve sample efficiency, though this area remains underdeveloped. Existing approaches include hand-crafted difficulty tiers \citep{wen2025lightr1curriculumsftdpo, deepscaler2025, fastcurl}, which often require task-specific insights and manual tuning. More adaptive methods, such as AdaRFT \citep{AdaRFT}, learn a training curriculum by dynamically adjusting a difficulty threshold to select examples. While promising, these methods still face limitations: repeated training on easy examples can lead to overfitting or poor generalization; difficulty heuristics may not transfer across tasks; and fixed sampling strategies may fail to adapt to evolving model capabilities. There remains a need for more principled curriculum strategies tailored to the scale and dynamics of LLM training.

\section{Self-paced Reinforcement Fine-tuning}
\begin{figure*}[t]
  \centering  
  \includegraphics[width=0.8\linewidth]{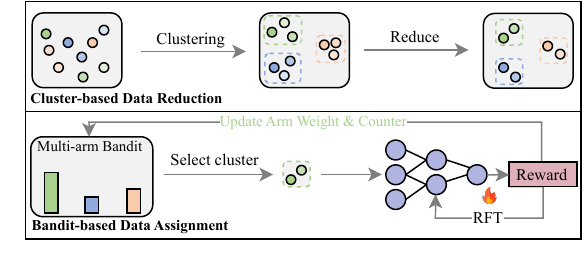}
    \caption{
      SPaCe Architecture. \textbf{Top:}
      In latent space, each initial training data's feature vector is formed by concatenating its latent embedding with its difficulty score. The data is clustered into \(K\) groups, and from each cluster we select representative samples to ensure both coverage and diversity. 
      \textbf{Bottom:}
      Each cluster is treated as an arm in a multi-armed bandit. At each step, Thompson Sampling is used to select a cluster, its representative examples are fed to the LLM to obtain rewards for RFT, and those rewards are used to update the bandit statistics.
    }
  \label{fig:SPaCe}
\end{figure*}

We aim to improve a policy $\pi_{\theta}$ by adaptively presenting training samples while minimizing the required data. Focusing on tasks that are too easy or too hard is inefficient, offering little challenge or feedback. Instead, data assignment should adapt to the model's evolving capabilities, presenting examples that it is ready to learn at each stage. Threshold-based curricula require manual tuning and can be suboptimal, especially at the early training stage~\citep{AdaRFT}. To address this, we introduce \textbf{SPaCe}, a two-phase approach for self-paced optimization: (1) cluster-based data reduction and (2) bandit-based data assignment. SPaCe integrates with common RFT algorithms; we use GRPO~\citep{deepseekai2025deepseekr1incentivizingreasoningcapability} by default.

\subsection{Cluster-based Data Reduction} \label{sec:SPaCe}
\subsubsection{Data Clustering}
\label{sec:clustering}
RFT of large language models often presupposes access to abundant, high-quality supervision, which is costly or impractical in low-resource settings. We introduce a clustering-based data reduction procedure that lowers data requirements while preserving or improving training efficacy. The approach leverages two signals per training instance: (i) a latent representation and (ii) a scalar, per-example attribute available from the dataset. Grouping examples that are proximate in latent space and exhibit comparable attribute values yields a coherent partitioning well-suited to curriculum learning.

\textit{Latent representation.} For each example, we obtain a latent representation using a pre-trained embedding model from SentenceTransformers \citep{reimers-2019-sentence-bert}. We decouple these clustering embeddings from the RL policy to avoid representation drift during RL training. To mitigate the degradation of distance metrics in high dimensions, we apply Principal Component Analysis (PCA) for dimensionality reduction \citep{pca}.

\textit{Per-example attribute.} Our framework supports the inclusion of an optional scalar attribute at the per-example level, alongside semantic embeddings. Let $d_i$ denote this attribute, which can flexibly encode any task-specific signal. For clustering, we concatenate $d_i$ with the latent embedding to form a joint representation. This design enables clustering to account not only for semantic similarity but also for structural or pedagogical cues that are important for curriculum construction.

\textit{Clustering.} Let $x_i$ denote the $i^\text{th}$ training example and $s_i$ its PCA-reduced latent vector. Before clustering, we standardize the coordinates of $s_i$ and $d_i$ to zero mean and unit variance, preventing either modality from dominating the distance metric. We then form the combined representation as follows:
\begin{equation}
    e_i \;=\; \hat{s}_i \oplus \hat{d}_i,
\end{equation}
where $\hat{s}_i$ and $\hat{d}_i$ are the standardized latent vector and scalar feature, respectively, and $\oplus$ denotes concatenation. Finally, we apply $k$-means clustering to $\{e_i\}$ to partition the dataset into $K$ clusters \citep{lloyd1982least}, thereby ensuring coverage across the joint latent–attribute space. We perform clustering only once before training to define a stable set of bandit arms. Fixing the arms ensures that each arm has a consistent meaning over time, so the scheduler can accumulate and compare per-arm statistics (e.g., solve rates and no-improvement counts) throughout training; re-clustering would change the arm identities and invalidate these statistics unless additional bookkeeping is introduced.

\subsubsection{Data Reduction}
Prior work on scaling RFT shows that carefully selected subsets of training data can outperform fine-tuning on the full corpus \citep{li2025limrrlscaling, wang2025reinforcementlearningreasoninglarge}. Motivated by this, the training set is partitioned into \(K\) clusters and a fixed quota \(l\) of examples is subsampled from each cluster to ensure balanced coverage and diversity. Let \(C_k\) denote the set of examples in cluster \(k\) with centroid \(\mu_k\) in the embedding space. For each example \(e_i \in C_k\), its distance to the centroid is computed as \(\delta_i = \lVert e_i - \mu_k \rVert_2\). Representative subsets are then selected via greedy farthest point sampling within each cluster: starting from the centroid, examples are iteratively chosen to maximize their distance from the already selected set. This preserves both centrality and geometric diversity within each cluster.

\subsection{Bandit-based Data Assignment}


Our key principle is to prioritize examples that are \emph{currently challenging and still learning-relevant}, since model capability evolves during training. At each step \(t\), a bandit \emph{scheduler} selects a cluster \(c_t\in\{1,\dots,K\}\) (arms correspond to clusters) using Thompson Sampling. For each cluster \(k\), we track cumulative reward \(R_k^{(t)}\) and pulls \(n_k^{(t)}\). We define the online solve rate as \(\bar r_k^{(t)} = R_k^{(t)}/(n_k^{(t)}+\epsilon)\), with \(\epsilon>0\) for numerical stability. This online rate is recomputed every step and is different from the offline difficulty \(d_i\) used only for clustering. Difficulty is defined as \(h_k^{(t)} := -\bar r_k^{(t)}\). Then, we use Thompson Sampling to draw a score \(\tilde\mu_k^{(t)}\) for each cluster and select the cluster with the largest draw:
\begin{equation}
\begin{aligned}
\tilde{\mu}_k^{(t)} &\sim \mathcal{N}\!\left( h_k^{(t)},\ \frac{1}{(n_k^{(t)}+\epsilon)}\right),\\
c_t &= \arg\max_{k\in\{1,\dots,K\}} \tilde{\mu}_k^{(t)}.
\end{aligned}
\end{equation}

In practice, we use a \emph{progress-aware} variant expressed directly in terms of hardness. For each cluster \(k\), we track whether its solve rate \(r_k^{(t)}\) remains \emph{informative} for learning: if \(r_k^{(t)}\) does not increase by at least a small tolerance for \(T_{\text{consecutive}}\) consecutive updates, we mark the cluster as \emph{stagnating} and subtract a small constant from its Thompson Sampling mean. This mechanism reduces the chance of repeatedly sampling clusters that remain hard but are no longer improving, thereby encouraging exploration among clusters of comparable hardness without overriding the standard Thompson Sampling trade-off.

This procedure naturally balances exploration and exploitation: clusters with fewer observations are more likely to be explored, while clusters with consistently low solve rates (i.e., harder samples) are sampled more frequently. A batch of size \(B\) is then drawn from cluster \(c_t\) to train \(\pi_\theta\). For each sample, we compute a binary correctness reward:
\begin{align}
r_i =
\begin{cases}
1, & \text{if the response is correct} \\
0, & \text{otherwise.}
\end{cases}
\end{align}

The average batch reward is \(r_{\text{avg}}=\frac{1}{B}\sum_{i=1}^{B} r_i\). This scalar signal updates the policy \(\pi_\theta\) via an RL algorithm (e.g., GRPO), by maximizing
\begin{equation}
\max_{\theta}\ \mathbb{E}_{q \sim D_{\text{train}},\, a \sim \pi_{\theta}}r_{\text{avg}},
\end{equation}
where \(a\) is the sampled answer from \(\pi_{\theta}\) given question \(q\). In addition, \(r_{\text{avg}}\) updates the bandit statistics for the selected cluster:
\begin{equation}
R_{c_t}^{(t+1)} = R_{c_t}^{(t)} + r_{\text{avg}}, \quad n_{c_t}^{(t+1)} = n_{c_t}^{(t)} + 1.
\end{equation}

This dynamic update ensures that the sampling distribution adapts online to the evolving state of the model: clusters that remain challenging receive more attention, while clusters that become easy are sampled less frequently. The full SPaCe procedure is provided in Algorithm~\ref{alg:algorithm} (Appendix~\ref{ref:algorithm}). We further analyze the convergence behavior of the scheduler in the following proposition.

\begin{proposition}
Under assumptions: (i) bounded rewards, LLM training with (ii) gradient clipping and (iii) decayed learning rate, the Thompson Sampling scheduler in SPaCe satisfies sublinear variation up to step $T$: \(V_T = O(\log T)\). Consequently, as \(t \to \infty\), the sampling distribution concentrates on clusters with maximal expected sampling score.
\end{proposition}

\begin{proof}
See Appendix~\ref{sec:bandit_convergence}.
\end{proof}

\section{Experiments}



\begin{table*}[t]
\centering
\small
\begin{tabular}{l|cccc|cccc}
\toprule
& \multicolumn{4}{c}{Qwen3-0.6B} & \multicolumn{4}{c}{DeepSeek-R1-Distill-Qwen-1.5B} \\
\cmidrule(lr){2-5}\cmidrule(lr){6-9}
Method& GSM8K & MATH500 & AIME24 & AIME25
& GSM8K & MATH500 & AIME24 & AIME25 \\
\midrule
Base    & 78.0 & 75.4 & 10.0 & \underline{16.7} & 70.2 & 78.4 & 20.0 & 16.7 \\
$\pi_1$ & 78.1 & 75.4 & \underline{13.3} & \underline{16.7} & 69.5 & 76.4 & 20.0 & 20.0 \\
$\pi_2$ & \underline{79.0} & 74.2 & 6.7 & 10.0 & 68.0 & 79.4 & 23.3 & 20.0 \\
Ordered & 77.9 & 74.8 & \underline{13.3} & 13.3 & \underline{71.0} & \underline{80.6} & \textbf{26.7} & 23.3 \\
SFT     & 77.6$_{0.5}$ & 74.8$_{0.4}$ & 7.8$_{1.9}$ & 14.4$_{5.1}$ & 70.4$_{0.9}$ & 79.7$_{0.8}$ & 15.6$_{1.9}$ & 18.9$_{1.9}$ \\
R1      & 77.9$_{1.0}$ & 71.6$_{1.3}$ & 8.9$_{5.1}$ & 12.2$_{5.1}$ & 70.7$_{1.1}$ & 80.2$_{0.9}$ & 18.9$_{1.9}$ & \underline{24.4$_{3.9}$} \\
AdaRFT  & 78.9$_{0.3}$ & \underline{75.9$_{1.5}$} & 12.2$_{3.9}$ & 14.4$_{3.9}$ & 69.9$_{0.2}$ & 76.5$_{1.5}$ & \textbf{26.7$_{3.3}$} & 16.7$_{6.7}$ \\
\rowcolor{gray!20}
\textbf{SPaCe} & \textbf{79.8$_{0.3}$} & \textbf{78.2$_{1.0}$} &
\textbf{20.0$_{5.8}$} & \textbf{18.9$_{5.1}$} &
\textbf{72.4$_{1.2}$} & \textbf{81.1$_{1.0}$} & \underline{25.6$_{3.9}$} & \textbf{26.7$_{5.8}$} \\
\bottomrule
\end{tabular}

\caption{Results with Qwen3-0.6B and DeepSeek-R1-Distill-Qwen-1.5B base LLMs trained on the DeepScaleR-Uniform dataset across multiple benchmarks. We report extractive match scores (\( \text{mean}_{\text{std}} \)) at the final training checkpoint, averaged over 3 seeds (except for the Base, \( \pi_1 \), \( \pi_2 \), and Ordered baselines). Best results are highlighted in bold, and the second-best are underlined.}
\label{tab:main_results}
\end{table*}

We evaluate the proposed method on multiple LLMs. Full fine-tuning is conducted on \textit{Qwen3-0.6B}, \textit{DeepSeek-R1-Distill-Qwen-1.5B}, \textit{Qwen2.5-0.5B-Instruct}, \textit{Falcon3-1B-Instruct}, and \textit{Llama3.2-1B-Instruct} using a single NVIDIA H100 GPU; notably, \textit{Qwen3-0.6B} attains performance comparable to larger models (e.g., \textit{Qwen2.5-Math-7B-Instruct}). We additionally fine-tune \textit{Qwen3-8B-Base} with LoRA \citep{hu2021loralowrankadaptationlarge} on a single H200 GPU. Experiments cover mathematical and logical reasoning benchmarks: DeepScaleR subsets (Uniform, Easy, and Difficult) (10k each) \citep{deepscaler2025}, GSM8K \citep{cobbe2021gsm8k}, and Knights and Knaves \citep{xie2024memorization}. Latent embeddings for clustering are extracted using \textit{Qwen3-Embedding-0.6B}. All datasets include difficulty annotations, either derived from solve rates of a moderate LLM \citep{AdaRFT} or provided as explicit labels \citep{xie2024memorization}; difficulty is used as the per-example clustering attribute due to its interpretability. The total number of training examples is \(l \times K\), where \(l\) is the per-cluster quota and \(K\) the number of clusters (reported in Table~\ref{tab:number_of_clusters}). Across all settings, \(K \leq 10\), yielding at most 100 training examples with \(l=10\). Each experiment is repeated with three random seeds and implemented using the Open-R1 codebase \citep{openr1}. The method introduces negligible runtime overhead relative to the R1 baseline; detailed timing results are provided in Figure~\ref{fig:trainingtime} in the appendix.

\paragraph{Evaluation}
We use five benchmarks that span different reasoning types and difficulty levels:  \textbf{GSM8K} \citep{cobbe2021gsm8k}, consists of diverse grade school math problems; \textbf{MATH500}, a 500-sample subset of the MATH dataset \citep{hendrycksmath2021}; \textbf{AIME24} and \textbf{AIME25}, comprising problems from the 2024 and 2025 American Invitational Mathematics Examination, respectively; and finally, the logical reasoning \textbf{K\&K} test set consists of 700 samples, with 100 examples for different number of people in the question from 2 to 8 \citep{xie2024memorization}. We report the extractive match scores for all mathematical datasets, following Lighteval's evaluation framework \citep{lighteval}. For K\&K dataset, we follow the evaluation protocols established by the dataset authors \citep{xie2024memorization}.

\paragraph{Baselines}
\textbf{Base} refers to the pretrained model without any fine‑tuning. \textbf{$\pi_1$} and \textbf{$\pi_2$} represent the baselines in \textbf{1-shot RLVR} paper, trained on one and two examples selected from the DeepScaleR dataset, respectively \citep{wang2025reinforcementlearningreasoninglarge}. \textbf{SFT} denotes the supervised fine‑tuning baseline. \textbf{Ordered} \citep{10.1145/1553374.1553380} is a curriculum baseline in which training begins with easier examples and gradually progresses to harder ones. \textbf{R1} is the RL baseline trained with the standard GRPO algorithm without an SFT cold start, as in DeepSeek-R1 \citep{deepseekai2025deepseekr1incentivizingreasoningcapability}. \textbf{AdaRFT} is a curriculum learning approach that selects examples based on a difficulty threshold \citep{AdaRFT}. We note that SFT, Ordered, R1, and AdaRFT baselines are trained on the full datasets. We also include the variance-based baseline \textbf{LIM} \citep{li2025limrrlscaling}, trained on MATH \citep{hendrycksmath2021}; details and results are in Appendix~\ref{sec:limr} due to its different training data.

\begin{figure*}[t]
  \centering
  \includegraphics[width=0.85\linewidth]{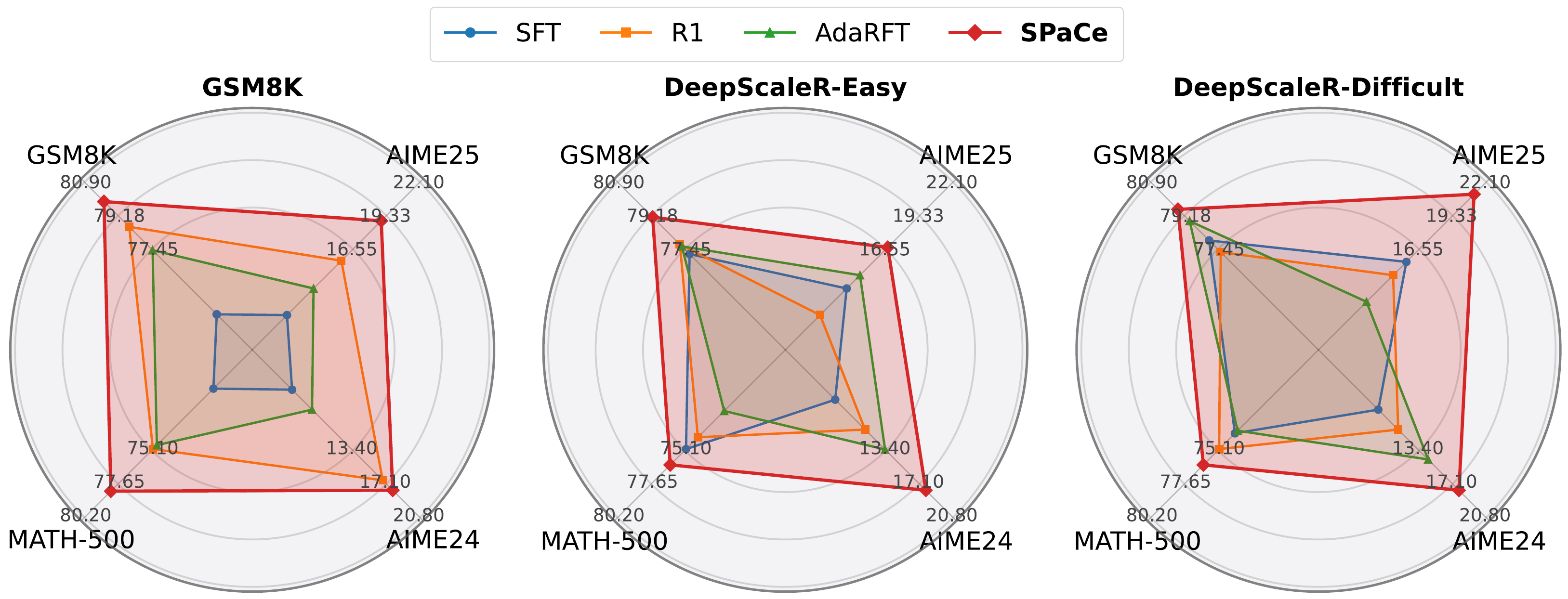}
    \caption{
      Results averaged over 3 training seeds using three training datasets with Qwen3-0.6B.
    }
  \label{fig:different_training_sets}
\end{figure*}

\section{Experimental Results}
\subsection{Mathematical Reasoning Benchmark} \label{sec:deepscaler_training}
We fine-tune \textit{Qwen3-0.6B} and \textit{DeepSeek-R1-Distill-Qwen-1.5B} in a zero-shot setting and report their performance on four mathematical benchmarks. Despite their small size, they serve as strong backbones for mathematical reasoning. Especially, \textit{Qwen3-0.6B} supports a \textit{thinking mode} enabled via \texttt{\textless think\textgreater} \texttt{\textless /think\textgreater} tags. Following~\citep{wang2025reinforcementlearningreasoninglarge}, we use a single seed for \(\pi_1\) and \(\pi_2\) baselines. For the \textit{Ordered} baseline, we likewise use one seed, since the training example order is fixed. Additional training hyperparameters are given in Appendix~\ref{appendix:SPaCe_details}.

Table~\ref{tab:main_results} presents test accuracies from the final training checkpoints using \textit{Qwen3-0.6B} and \textit{DeepSeek-R1-Distill-Qwen-1.5B}. SPaCe consistently achieves the highest performance across all benchmarks. On \textit{Qwen3-0.6B}, SPaCe attains the best score on all four benchmarks, improving over the Base model by $1.8\%$ on GSM8K and $2.8\%$ on MATH500, with larger gains on the harder AIME sets ($+5.6\%$ on AIME24 and $+2.2\%$ on AIME25). On \textit{DeepSeek-R1-Distill-Qwen-1.5B}, SPaCe improves over Base on GSM8K ($+2.2\%$), MATH500 ($+2.7\%$), AIME24 ($+5.6\%$), and AIME25 ($+10.0\%$), achieving the best result on three out of four benchmarks and remaining within $1.1\%$ of the best method on AIME24. Overall, SPaCe provides consistent improvements relative to the included baselines, with particularly noticeable gains on the AIME benchmarks.

\subsection{SPaCe Works With Various Datasets} \label{sec:gsm8k_training}
\subsubsection{Mathematical Datasets Training Results} We evaluate our method on three distinct training sets: (1) GSM8K; (2) DeepScaleR–Easy, a subset of DeepScaleR with primarily low‑difficulty questions; and (3) DeepScaleR–Difficult, a subset with mainly high‑difficulty questions. All experiments use \textit{Qwen3-0.6B} as the backbone, and we adopt the same reward functions and hyperparameters as in Section~\ref{sec:deepscaler_training}.We compare against the top-3 baselines, excluding DeepScaleR-specific baselines $\pi_1$ and $\pi_2$. As seen in Figure~\ref{fig:different_training_sets}, across all settings, SPaCe consistently outperforms standard SFT, achieving gains of 2.9–5.5 \% on GSM8K and up to 7.8 \% on the AIME benchmarks. R1 generally ranks second, especially on the easier splits, while AdaRFT falls 1–2 \% behind in most cases. Notably, when trained on the difficult subset (3), SPaCe attains a 2.3 \% improvement on MATH500 and more than doubles AIME24 accuracy relative to SFT. These results confirm that SPaCe not only enhances overall accuracy but also yields the greatest benefits on the most challenging training set.

\begin{table*}[t]
\centering
\setlength{\tabcolsep}{3pt}
\small
\begin{tabular}{lccccccc|c}
\toprule
& \multicolumn{7}{c}{\textbf{Number of People}} & \textbf{Average} \\
\cmidrule(lr){2-8}
\textbf{Method} & \textbf{2} & \textbf{3} & \textbf{4} & \textbf{5} & \textbf{6} & \textbf{7} & \textbf{8} & \\
\midrule
Base & 32.0 & 10.0 & 8.0 & 2.0 & 0.0 & 0.0 & 0.0 & 7.4 \\

R1 & 31.7$\pm$1.5 & 11.3$\pm$0.6 & 8.3$\pm$0.6 & 4.0$\pm$1.0 & 0.7$\pm$0.6 & 0.0$\pm$0.0 & 1.0$\pm$0.0 & 8.1$\pm$0.3 \\

AdaRFT & 31.7$\pm$0.6 & 10.7$\pm$2.1 & 7.3$\pm$0.6 & 3.7$\pm$0.6 & 0.3$\pm$0.6 & 0.0$\pm$0.0 & 1.0$\pm$0.0 & 7.8$\pm$0.2 \\

\rowcolor{gray!20}
\textbf{SPaCe} & \textbf{34.3$\pm$1.2} & \textbf{15.7$\pm$3.1} & \textbf{10.7$\pm$1.5} & \textbf{5.7$\pm$0.6} & \textbf{1.7$\pm$0.6} & 0.0$\pm$0.0 & \textbf{1.3$\pm$0.6} & \textbf{9.9$\pm$1.0} \\
\bottomrule
\end{tabular}
\caption{Accuracy (\%) by number of people in K\&K puzzles with results reported as mean $\pm$ standard deviation (except for Base baseline) over 3 runs using \textit{Qwen3-0.6B} as the base model. Bold denotes the best mean performance.}
\label{tab:kk-difficulty}
\end{table*}

\begin{figure*}[t]
  \centering
  \includegraphics[width=\linewidth]{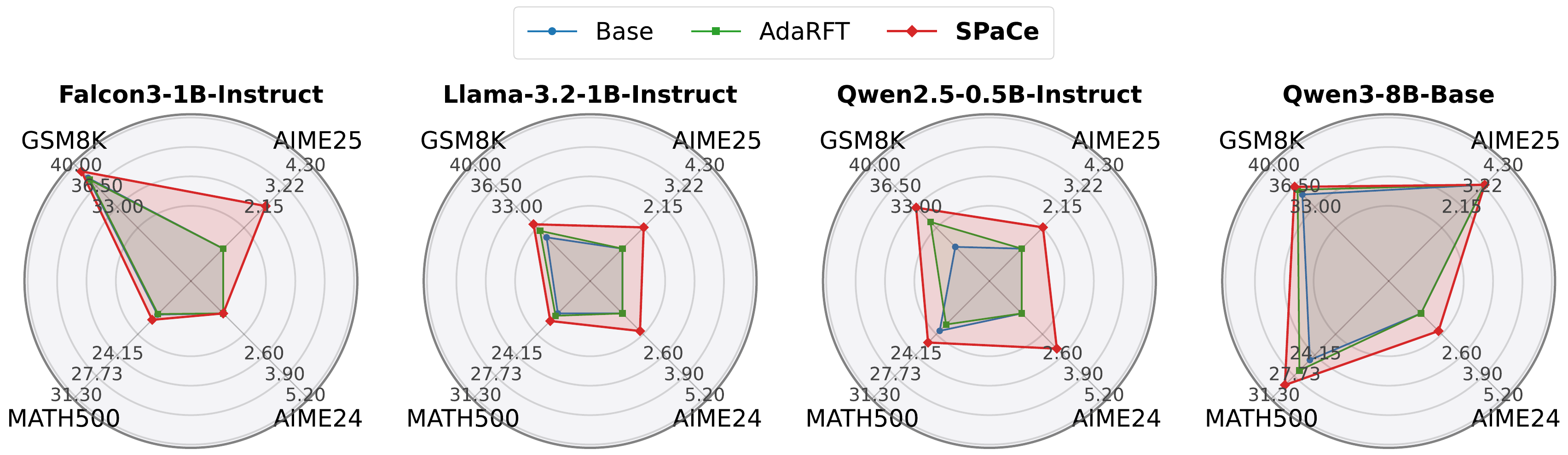}
    \caption{
      Results averaged over 3 training seeds using other LLMs across datasets. 
    }
  \label{fig:different_llms}
\end{figure*}

\subsubsection{K\&K Training Results}

In this dataset, we consider the number of people in each question as the per-example attribute \(d_i\) and use \textit{Qwen3-0.6B} as the base LLM. The final answer is used to compute the accuracy reward, and the original evaluation protocol \citep{xie2024memorization} is followed for consistency and comparability. As shown in Table~\ref{tab:kk-difficulty}, SPaCe consistently outperforms all baselines across difficulty levels, with especially clear gains on harder puzzles with 2--5 people, where reasoning demands are higher. It achieves the highest overall accuracy of 9.9\%, versus 8.1\% for R1 and 7.8\% for AdaRFT, which shows relative improvements of $22.2\%$ and $26.9\%$, respectively. AdaRFT’s accuracy drops in 2 of 7 settings, suggesting that a noisy curriculum can harm performance. Overall, these results show that SPaCe scales effectively to harder reasoning cases while remaining competitive on simpler ones, validating curriculum-guided selection for reasoning-focused training.

\subsection{SPaCe Helps Diverse LLM Learners} \label{sec:extension_llms}
We train on DeepScaleR-Uniform and evaluate \textit{Qwen2.5-0.5B-Instruct}, \textit{Falcon3-1B-Instruct}, \textit{Llama3.2-1B-Instruct}, and \textit{Qwen3-8B-Base}, which are compact to mid-size LLMs with strong reasoning, language, code, and math skills. We exclude \textit{Qwen3-8B} (reasoning-enabled) due to the substantial compute from long \texttt{<think>} traces. We compare against the Base model and AdaRFT, the most consistent and second-best method in Section~\ref{sec:deepscaler_training}. All models use the same zero-shot setup, except \textit{Llama3.2-1B-Instruct}, which requires one in-context example per instance to yield valid correctness rewards \citep{le2025reasoning}. Figure~\ref{fig:different_llms} shows SPaCe as the clear winner. Average gains are evident on GSM8K (\mbox{$\sim$+3}\%) and MATH500 (\mbox{$\sim$+2}\%). On the harder AIME splits, SPaCe turns near-zero baseline scores into consistent positives, reflecting better sample efficiency under sparse-reward RFT. We see complete or near-complete sweeps on the smaller models over Base and AdaRFT, indicating benefits in capacity-constrained settings; results on \textit{Qwen3-8B-Base} remain strong despite not being used for RFT training. Overall, a bandit-driven, performance-aware curriculum generalizes across architectures and tasks with minimal protocol changes, delivering reliable gains under compute-conscious budgets.

\begin{figure*}[t]
    \centering
    \includegraphics[width=\linewidth]{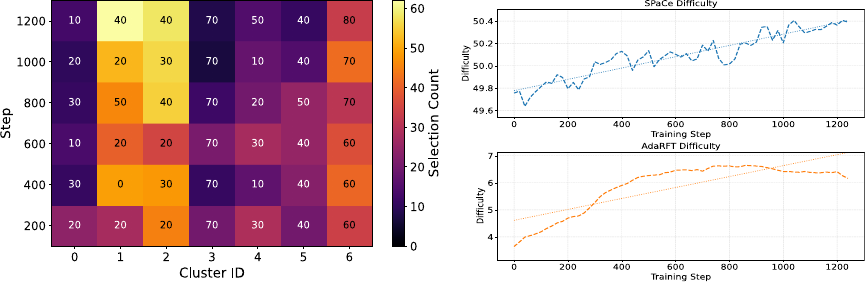}
    \caption{\textbf{Left:} Multi-arm bandit cluster selection (heatmap) with per-cluster solve rates (percentages) annotated in each cell across training steps. \textbf{Right:} \textbf{Top:} Difficulty of examples selected by SPaCe over time. \textbf{Bottom:} Difficulty of examples selected by AdaRFT over time.}
    \label{fig:selected_question_difficulties}
\end{figure*}



\section{Ablation Studies and Model Analyses}
\subsection{Multi-arm Bandit Analysis} 
\paragraph{Empirical Convergence}
SPaCe leverages the MAB framework to adaptively guide curriculum learning, making it important to characterize how the scheduler evolves during training. Figure~\ref{fig:selected_question_difficulties} (Left) illustrates a heatmap of cluster solve rates alongside bandit selections over time using \textit{Qwen3-0.6B} on the DeepScaleR-Uniform dataset partitioned into 7 clusters. Early in training (around step 200), the bandit behaves nearly uniformly, allocating samples across clusters with little preference. As the model accumulates experience, clear patterns emerge: clusters 1 and 2 exhibit higher solve-rate differentials, and the bandit correspondingly shifts toward sampling them more frequently, signaling that these clusters provide greater marginal learning benefit. From step 600 onward, this concentration intensifies, with clusters 1 and 2 dominating the selection distribution, indicating that the scheduler successfully adapts to focus training on regions of the data that remain most informative for continued performance improvement.

\paragraph{Solve Rate Trends.}
Clusters 1 and 2 deliver the largest gains, with solve rate rising from 20\% at step 200 to 40\% by step 1200, indicating moderate difficulty and strong learning signals. In contrast, cluster 3 is rarely selected and stays flat at around 70\%, suggesting it is too easy to drive improvement. Clusters 4 and 5 show smaller gains under continued exploration, while cluster 0 provides little benefit. By the end of training, the bandit concentrates on clusters 1 and 2, which consistently yield the highest returns.

\begin{table}[t]
\centering
\setlength{\tabcolsep}{3pt}
\setlength{\arrayrulewidth}{0.4pt}
\small
\resizebox{\columnwidth}{!}{
\begin{tabular}{l|cccc}
\toprule
Method 
& GSM8K & MATH500 & AIME24 & AIME25 \\
\midrule
Base & 78.0 & 75.4 & 10.0 & 16.7\\
SPaCe$^{-}$ 
& 78.8$\pm$0.4 & 76.4$\pm$0.7 & 12.2$\pm$3.9 & 11.1$\pm$1.9 \\
\rowcolor{gray!20}
SPaCe 
& \textbf{79.8$\pm$0.3} & \textbf{78.2$\pm$1.0} & \textbf{20.0$\pm$5.8} & \textbf{18.9$\pm$5.1} \\
\bottomrule
\end{tabular}
}
\caption{
Mean~$\pm$~std over 3 seeds on DeepScaleR-uniform using Qwen3-0.6B.
SPaCe$^{-}$ denotes the variant without data reduction.
Best results in bold.
}
\label{tab:data_selection_impact}
\end{table}

\subsection{Data Reduction Analysis}
\paragraph{Average Sample Difficulty Comparison} 
We consider \textit{Qwen3-0.6B} on DeepScaleR-Uniform with 7 clusters to examine selected example difficulties. Figure~\ref{fig:selected_question_difficulties} (Right) compares the average difficulty of training examples chosen over time by SPaCe and AdaRFT. Although AdaRFT accesses the full dataset, its performance with LLMs struggles to reach medium and hard examples, due to a threshold mechanism that is highly sensitive and skews selection. In contrast, SPaCe favors medium to hard instances, avoiding the overemphasis on easy examples seen in AdaRFT. We attribute this to our clustering strategy, which captures both semantic diversity and difficulty: each cluster mixes a broad range of examples, enabling exploration of harder cases without sacrificing variety.

\paragraph{Impact of Data Reduction}
To assess the data reduction phase, we ablate selection by keeping all examples in each cluster. As shown in Table~\ref{tab:data_selection_impact}, this variant (SPaCe$^{-}$) still outperforms the Base baseline but underperforms full SPaCe. Without reduction, the number of examples per cluster can exceed the batch size \(B\), increasing variance and within-batch difficulty heterogeneity, which weakens the learning signal.

\begin{figure*}[t]  
    \centering
    \includegraphics[width=0.8\linewidth]{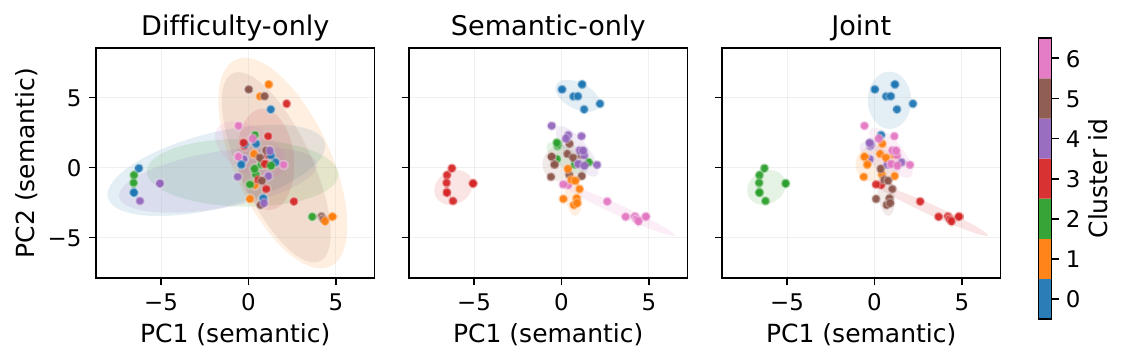}
    \caption{Clustering variants in semantic space with 7 clusters using DeepScaleR-Uniform.}
    \label{fig:clustering_effects}
\end{figure*}

\begin{figure*}[t]  
    \centering
    \includegraphics[width=0.8\linewidth]{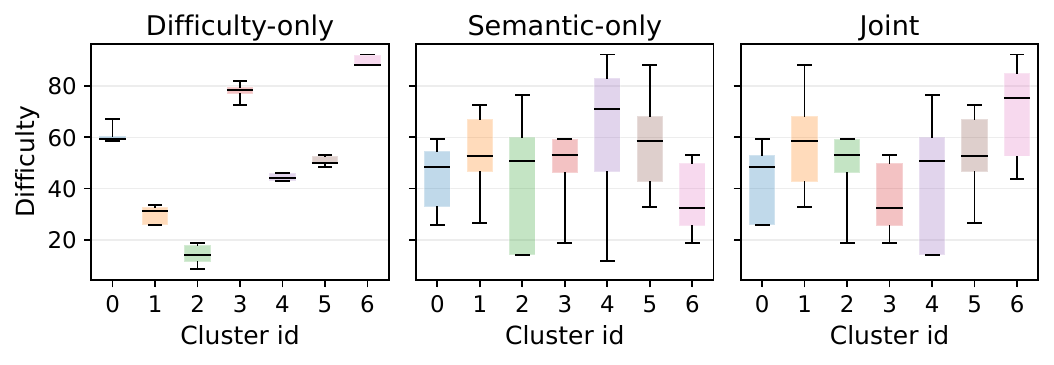}
    \caption{Difficulty of three clustering variants in semantic space with 7 clusters using DeepScaleR-Uniform.}
    \label{fig:difficulty_boxplot}
\end{figure*}

\subsection{Clustering Effects} \label{sec:clustering_effects}
\paragraph{Clustering Analysis}
We further investigate the effect of training-data clustering. Figure~\ref{fig:clustering_effects} visualizes the resulting clusters in a shared semantic space and highlights clear differences across clustering strategies. When clustering is based only on difficulty, examples with very different semantic content are often grouped together simply because they share similar hardness levels. This produces clusters that are less coherent in terms of task or content structure. In contrast, semantic-only clustering groups together examples that are topically similar, but the corresponding hardness regions remain highly overlapping, making it difficult to separate clusters by learning difficulty alone. Joint clustering, which combines semantic representations with per-example difficulty, achieves a better balance: the resulting clusters remain semantically localized while also exhibiting clearer separation along the difficulty dimension.

Figure~\ref{fig:difficulty_boxplot} further supports this observation by showing the distribution of difficulty scores within each cluster. Semantic-only clusters tend to span broad and overlapping difficulty ranges, suggesting that semantic similarity alone is insufficient for forming well-structured curriculum units. By comparison, joint clustering produces tighter and more coherent difficulty bands, while still avoiding the degenerate behavior of purely difficulty-based clustering, where semantic diversity within each cluster becomes too large. Taken together, these findings suggest that concatenating embeddings with per-example difficulty provides a principled way to construct curriculum units that preserve semantic similarity while also controlling variation in difficulty. This balance is particularly desirable for curriculum learning, where both semantic coherence and gradual difficulty progression are important for effective training.

\subsection{Other Ablation Studies} 
We also ablate (i) paired statistical testing, (ii) number of clusters, (iii) diverse sample selection, (iv) embedding model choice, (v) number of PCA components, (vi) samples per cluster, (vii) removing difficulty effect, (viii) selected-sample difficulty, (ix) dataset distribution, (x) cluster properties and (xi) training time (Appendices~\ref{sec:paired_stats}, \ref{sec:number_clusters}, \ref{sec:diverse_sample_selection}, \ref{sec:embedding_model_choice}, \ref{sec:number_pca_component}, \ref{sec:number_samples_each}, \ref{sec:remove_difficulty}, \ref{sec:selected_training_samples}, \ref{sec:datset_dist}, \ref{sec:cluster_analysis}, \ref{sec:trainingtime}). Across these dimensions, the results consistently support the robustness of our approach, clarify trade-offs and hyperparameter sensitivities, and offer practical guidance for default settings.

\section{Conclusion}
We introduced \textbf{SPaCe}, a lightweight framework that enables efficient reasoning in language models through clustering and adaptive curriculum learning. SPaCe selects compact, diverse training subsets and dynamically adapts training focus based on model performance. Experiments show that SPaCe achieves competitive accuracy with significantly fewer samples. These results highlight the effectiveness of combining semantic clustering with performance-driven curricula to unlock reasoning in small models using minimal resources.

\section*{Limitations}
While \textsc{SPaCe} consistently improves training efficiency and performance across our evaluated settings, our experiments are limited to models with fewer than 8B parameters. Further evaluation on larger models would be valuable to better assess scalability and generality.

\bibliography{custom}
\newpage

\appendix
\section{Appendix}
\label{appendix}
\subsection{Algorithm for SPaCe} \label{ref:algorithm}
In this section, we provide the pseudo-code for SPaCe in Algorithm \ref{alg:algorithm}.

\begin{algorithm}[h]
\caption{SPaCe}
\label{alg:algorithm}
\textbf{Input}: Policy \(\pi_\theta\), Dataset \(\mathcal{D}\), Embedding model \(\phi\), Clusters \(K\), Batch size \(B\), RL algorithm \(\mathcal{A}\), \(\epsilon>0\), \(\delta>0\), \(T_{\text{consecutive}}\), \(\gamma>0\) \\
\textbf{Output}: Trained policy \(\pi_\theta\)
\begin{algorithmic}[1]
\STATE \textbf{// Phase 1: Cluster-based Data Reduction}
\FOR{each \(x_i \in \mathcal{D}\)}
    \STATE \(e_i \leftarrow \text{PCA}(\phi(x_i)) \oplus \text{difficulty}(x_i)\)
\ENDFOR
\STATE Run K-means on \(\{e_i\}\) to form \(\{C_k\}_{k=1}^K\); pick \(l\) diverse samples per cluster \(\Rightarrow \mathcal{D}_{\text{train}}\)

\STATE \textbf{// Phase 2: Bandit-driven Curriculum}
\FOR{each \(k=1,\dots,K\)} \STATE \(R_k\!\leftarrow\!0,\ n_k\!\leftarrow\!0,\ \texttt{no\_improve}_k\!\leftarrow\!0\) \ENDFOR
\WHILE{training not finished}
    \FOR{each \(k=1,\dots,K\)}
        \STATE \(h_k \leftarrow -\frac{R_k}{n_k+\epsilon}\) 
        \STATE \(m_k \leftarrow h_k - \gamma\cdot\mathbb{I}[\texttt{no\_improve}_k \ge T_{\text{consecutive}}]\)
        \STATE \(\tilde{\mu}_k \sim \mathcal{N}\!\left(m_k,\ \frac{1}{n_k+\epsilon}\right)\)
    \ENDFOR
    \STATE \(c_t \leftarrow \arg\max_k \tilde{\mu}_k\); sample \(X \subset C_{c_t}\), \(|X|=B\)
    \STATE \(G \leftarrow \pi_\theta(X)\); compute \(r_{\text{avg}}=\frac{1}{B}\sum_{i=1}^B r_i\), \(r_i\in\{0,1\}\)
    \STATE \(\pi_\theta \leftarrow \mathcal{A}(\pi_\theta, X, G, r_{\text{avg}})\)
    \STATE \(R_{c_t} \leftarrow R_{c_t}+r_{\text{avg}},\quad n_{c_t} \leftarrow n_{c_t}+1\)
    \STATE \(\texttt{no\_improve}_{c_t} \leftarrow \mathbb{I}\!\left[\frac{R_{c_t}}{n_{c_t}+\epsilon} < \frac{R_{c_t}-r_{\text{avg}}}{(n_{c_t}-1)+\epsilon} + \delta\right]\)
\ENDWHILE
\STATE \textbf{return} \(\pi_\theta\)
\end{algorithmic}
\end{algorithm}

\subsection{Convergence of the Thompson Sampling Scheduler}
\label{sec:bandit_convergence}

We analyze the convergence of the Thompson Sampling scheduler used in SPaCe. Each data cluster is treated as an arm in a multi-armed bandit. At step \( t \), let \( \pi_{\theta^{(t)}} \) denote the model with parameters \( \theta^{(t)} \). The expected reward (solve rate) of cluster \( \mathcal{C}_k \) is defined as:
\begin{equation}
\mu_k^{(t)} = \mathbb{E}_{x \sim \mathcal{C}_k} \left[\Pr(\pi_{\theta^{(t)}}(x) =\text{ correct})\right].    
\end{equation}
where \( \Pr\left(\pi_{\theta^{(t)}}(x) = \text{correct}\right) \) denotes the probability that the model produces a correct answer for input \( x \).

We already have: (1) the model is trained using gradient clipping with threshold \( G_{\max} \); (2) the learning rate \( \alpha_t \) follows a cosine decay schedule with warmup and is therefore non-increasing and vanishes as \( t \to \infty \); and (3) the expected rewards satisfy \( \mu_k^{(t)} \in [0,1] \) for all clusters \( k \). These three properties are directly enforced in the SPaCe implementation.

To complete the convergence analysis, we now bound the drift of each cluster’s expected reward.  Define
\begin{equation}
f_k(\theta)\;=\;\mathbb{E}_{x \sim \mathcal{C}_k}\bigl[\Pr\bigl(\pi_\theta(x)\ =\text{correct}\bigr)\bigr],  
\end{equation}

where \(f_k\) is the \emph{cluster-level reward surface} for arm \(k\), and note that since each layer of our base model is continuously differentiable, so is \(f_k(\theta)\) \citep{Goodfellow-et-al-2016}. We assume that, in practice, gradient clipping at norm \(G_{\max}\) together with a bounded initialization prevents the parameters \(\{\theta^{(t)}\}\) from diverging excessively, effectively keeping them in some large but fixed ball \(\{\|\theta\|\le R\}\).  Empirical studies on LLM training have repeatedly observed that clipped updates under cosine‑decay schedules yield stable trajectories without catastrophic parameter growth \citep{wang2025adagcimprovingtrainingstability, huang2025spam}.  

Under this assumption, the extreme‐value theorem \citep{rudin1976principles} guarantees the existence of a constant \(H<\infty\) such that
\begin{equation}    
\|\nabla_\theta f_k(\theta)\|\;\le\; H
\quad\forall\ \|\theta\|\le R.
\end{equation}
Moreover, each gradient step with cosine‑decay learning rate \(\alpha_t\) and gradient clipping satisfies
\begin{equation}    
\|\theta^{(t+1)} - \theta^{(t)}\|\;\le\;\alpha_t\,G_{\max}.
\end{equation}
Applying the mean‐value theorem \citep{rudin1976principles} then yields
\begin{align}
\bigl|\mu_k^{(t+1)} - \mu_k^{(t)}\bigr|
&=\bigl|f_k(\theta^{(t+1)}) - f_k(\theta^{(t)})\bigr|\nonumber\\
&\le H\,\|\theta^{(t+1)} - \theta^{(t)}\|\nonumber\\
&\le H\,G_{\max}\,\alpha_t
\;=\;\varepsilon_t,
\end{align}
where \(\varepsilon_t\to0\) as \(\alpha_t\to0\).  Thus, we obtain the desired vanishing drift
\(\bigl|\mu_k^{(t+1)} - \mu_k^{(t)}\bigr|\le\varepsilon_t\)
for every cluster \(k\).  

Let \( V_T = \sum_{t=1}^{T-1} \max_k |\mu_k^{(t+1)} - \mu_k^{(t)}| \) denote the total reward variation up to step \( T \). The bound above implies
\begin{equation}    
V_T \le \sum_{t=1}^{T-1} \varepsilon_t.
\end{equation}

Let \(A_T := \sum_{t=1}^{T-1}\alpha_t\) denote the cumulative step size up to step \(T\). Combining the drift bound with the definition of \(V_T\) gives
\begin{equation}
V_T \;\le\; \sum_{t=1}^{T-1}\varepsilon_t
\;=\; H\,G_{\max}\sum_{t=1}^{T-1}\alpha_t
\;=\; H\,G_{\max}\,A_T.
\end{equation}
For the cosine decay schedule with warmup used in SPaCe, the learning rate is non-increasing after warmup and reaches \(0\) at the end of the prescribed training horizon \(T_{\text{train}}\); we can equivalently take \(\alpha_t=0\) for all \(t>T_{\text{train}}\). Hence \(A_T\) is finite and bounded by
\begin{equation}
A_T \;\le\; \sum_{t=1}^{T_{\text{train}}}\alpha_t \;=\; O(T_{\text{train}}\alpha_{\max}),
\end{equation}
which implies that the total reward variation \(V_T\) is also bounded (and therefore sublinear as \(T\to\infty\) under the extension \(\alpha_t=0\) for \(t>T_{\text{train}}\)). This bounded variation budget matches the standard non-stationary bandit setting studied in prior work \citep{besbes2014stochastic}, and supports the use of Thompson Sampling as a scheduler that tracks the best-performing cluster as training progresses and the drift \(\varepsilon_t\) becomes small. In particular, as \(\alpha_t\to 0\), we have \(\varepsilon_t\to 0\), so the cluster rewards become effectively stationary, and Thompson Sampling concentrates its selections on the cluster(s) with the highest current expected reward.

\subsection{Additional Results}

In this section, we provide the full results for Qwen3-8B-Base in Table \ref{tab:qwen3_8b_base} and Meta-Llama3-8B in Table \ref{tab:llama3_8b}. These new experiments confirm that the gains of our method persist when scaling to the 8B regime, addressing the concern that our findings may be specific to small models.

\begin{table*}[t]
\centering
\caption{Qwen3-8B-Base results (accuracy $\pm$ std).} 
\label{tab:qwen3_8b_base}
\begin{tabular}{lcccc}
\hline
Method & GSM8K & MATH500 & AIME24 & AIME25 \\
\hline
Base & 35.1 & 25.0 & 0.0 & \textbf{3.3} \\
Method A & 35.0 & 28.0 & \textbf{1.1} & 0.0 \\
Method B & 34.7 & 28.2 & 0.0 & 0.0 \\
Ordered & 34.6 & 26.0 & 0.0 & 0.0 \\
SFT & 34.0 $\pm$ 0.7 & 26.4 $\pm$ 0.6 & 0.0 $\pm$ 0.0 & 0.0 $\pm$ 0.0 \\
R1 & 35.5 $\pm$ 1.3 & 27.0 $\pm$ 0.6 & \textbf{1.1 $\pm$ 1.9} & 0.0 $\pm$ 0.0 \\
AdaRFT & 35.9 $\pm$ 0.7 & 26.8 $\pm$ 1.2 & 0.0 $\pm$ 0.0 & \textbf{3.3 $\pm$ 0.0} \\
\textbf{SPaCe (Ours)} & \textbf{36.4 $\pm$ 0.9} & \textbf{29.3 $\pm$ 1.1} & \textbf{1.1 $\pm$ 1.9} &\textbf{3.3 $\pm$ 0.0} \\
\hline
\end{tabular}
\end{table*}

\begin{table*}[t]
\centering
\caption{Llama3-8B results (accuracy $\pm$ std).}
\label{tab:llama3_8b}
\begin{tabular}{lcccc}
\hline
Method & GSM8K & MATH500 & AIME24 & AIME25 \\
\hline
Base & 34.4 & 26.4 & 0.0 & 3.3 \\
Method A & 50.3 & 30.3 & \textbf{3.3} & 0.0 \\
Method B & 51.0 & 31.1 & \textbf{3.3} & 0.0 \\
Ordered & 50.6 & 32.2 & 0.0 & 0.0 \\
SFT & 50.3 $\pm$ 1.2 & 38.4 $\pm$ 0.9 & \textbf{3.3 $\pm$ 0.0} & 0.0 $\pm$ 0.0 \\
R1 & 51.6 $\pm$ 1.5 & 34.0 $\pm$ 1.6 & 0.0 $\pm$ 0.0 & 0.0 $\pm$ 0.0 \\
AdaRFT & 51.2 $\pm$ 0.8 & 36.1 $\pm$ 1.2 & \textbf{3.3 $\pm$ 0.0} & 3.3 $\pm$ 0.0 \\
\textbf{SPaCe (Ours)} & \textbf{52.2 $\pm$ 1.3} & \textbf{38.8 $\pm$ 0.7} & \textbf{3.3 $\pm$ 0.0} & \textbf{5.5 $\pm$ 3.9} \\
\hline
\end{tabular}
\end{table*}

\subsection{Paired Statistical Analysis} \label{sec:paired_stats}
To verify the impact of SPaCe, we conduct paired statistical tests between SPaCe and baselines, each run with three random seeds. We use matched training seeds for SPaCe and the seed-controlled baselines (SFT, R1, and \textsc{AdaRFT}) to enable a paired comparison. We report paired accuracy improvements (in percentage) of SPaCe over each baseline across multiple settings, using the Student's $t$-distribution with $df{=}2$ \citep{student1908probable}. The comparisons using other LLMs are shown in Table \ref{tab:paired_both_backbones}. The comparisons on Knights and Knaves are shown in Table \ref{tab:kk_paired}. The comparisons using other datasets as training data are shown in Table \ref{tab:paired_deepscaler_splits}.

Table~\ref{tab:paired_both_backbones} reports paired accuracy gains of SPaCe over each baseline across three shared seeds. On \textbf{Qwen3-0.6B}, SPaCe yields \emph{consistently significant} improvements on \textbf{GSM8K} and \textbf{MATH500}, with all 90\% confidence intervals remaining strictly above zero, indicating reliable gains across baselines. On \textbf{DeepSeek-R1}, improvements are also robust on \textbf{MATH500}, while \textbf{GSM8K} gains are positive but less stable due to wider intervals. In contrast, results on \textbf{AIME24} and \textbf{AIME25} exhibit substantially larger variance and several intervals overlap zero, suggesting that while mean gains are often positive, statistical evidence is inconclusive under the three-seed setting.

On Knights and Knaves (Table~\ref{tab:kk_paired}), SPaCe consistently matches or improves upon both R1 and AdaRFT across depths 2--8, with the strongest mean gains concentrated at shallower depths (2--4). While confidence intervals are wider under three seeds, the improvements remain directionally positive throughout and become more stable at higher depths, where results converge as the task becomes more constrained. Overall, these findings suggest that SPaCe provides reliable benefits on this dataset, particularly for lower-depth instances where effective curriculum selection is most impactful.

Table~\ref{tab:paired_deepscaler_splits} shows that SPaCe improves over both R1 and AdaRFT on DeepScaleR-Easy and DeepScaleR-Difficult, with consistently positive mean gains across all benchmarks. On the Easy split, improvements are particularly stable on \textbf{MATH500} (e.g., $+2.13\pm0.00$ over R1 and $+3.40\pm0.78$ over AdaRFT), indicating reliable benefits under matched seeds. Notably, SPaCe continues to provide meaningful gains on the Difficult split, including improvements on \textbf{GSM8K} and \textbf{MATH500} and sizeable gains on \textbf{AIME25} (up to $+6.67\pm4.89$ over R1). While the hardest benchmarks (AIME24/AIME25) exhibit larger variability under three seeds, the overall trend remains positive, suggesting that SPaCe generalizes across both easier and more challenging training regimes.

Finally, Table \ref{tab:paired_other_LLMs_uniform_vs_adarft} shows results using other LLMs as backbones. Across four backbones, SPaCe consistently yields positive paired accuracy improvements over AdaRFT, indicating that its data assignment strategy generalizes well beyond a single base LLM. The gains are most pronounced on \textbf{GSM8K}, where every backbone benefits from a clear uplift, suggesting improved reasoning reliability under shared-seed evaluation. On \textbf{MATH500}, SPaCe remains beneficial and can be notably stable, with \textbf{Qwen3-8B-Base} showing a strong improvement accompanied by a tight confidence interval, consistent with dependable rather than noisy gains. Even on the more challenging \textbf{AIME} benchmarks, SPaCe maintains non-degrading behavior and achieves additional improvements on several backbones, reinforcing its robustness and transferability across diverse model families.

\begin{table*}[t]
\centering
\caption{Paired accuracy improvements (\%) of \textbf{SPaCe} over each baseline, shown for two backbones.
Entries are mean $\pm$ 90\% confidence interval over three shared random seeds (Student's $t$-distribution, $df=2$).}
\label{tab:paired_both_backbones}
\resizebox{\textwidth}{!}{%
\begin{tabular}{l|cccc|cccc}
\toprule
& \multicolumn{4}{c}{\textbf{Qwen3-0.6B}} & \multicolumn{4}{c}{\textbf{DeepSeek-R1}} \\
\cmidrule(lr){2-5}\cmidrule(lr){6-9}
\textbf{Baseline}
& \textbf{GSM8K} & \textbf{MATH500} & \textbf{AIME24} & \textbf{AIME25}
& \textbf{GSM8K} & \textbf{MATH500} & \textbf{AIME24} & \textbf{AIME25} \\
\midrule
SFT
& $+2.18 \pm 1.09$ & $+3.39 \pm 0.99$ & $+10.01 \pm 4.10$ & $+4.44 \pm 9.77$
& $+2.01 \pm 3.32$ & $+1.33 \pm 0.85$ & $+10.00 \pm 9.75$ & $+7.77 \pm 8.59$ \\

R1
& $+1.79 \pm 1.66$ & $+6.56 \pm 1.94$ & $+8.89 \pm 10.12$ & $+6.66 \pm 6.94$
& $+1.73 \pm 1.81$ & $+0.87 \pm 0.50$ & $+6.66 \pm 5.64$ & $+2.22 \pm 8.59$ \\

AdaRFT
& $+0.83 \pm 0.42$ & $+2.26 \pm 2.01$ & $+5.56 \pm 6.03$ & $+4.44 \pm 2.92$
& $+2.54 \pm 1.72$ & $+4.53 \pm 4.12$ & $-1.11 \pm 8.59$ & $+10.00 \pm 20.30$ \\
\bottomrule
\end{tabular}%
}
\end{table*}

\begin{table*}[t]
\centering
\small
\caption{Paired improvements of \textbf{SPaCe} over baselines on the Knights and Knaves dataset using Qwen3-0.6B as the base LLM. Entries are mean $\pm$ 90\% confidence interval over three shared random seeds (Student's $t$-distribution, $df=2$).}
\label{tab:kk_paired}
\resizebox{\textwidth}{!}{%
\begin{tabular}{l|ccccccc}

\toprule
Baseline & 2 & 3 & 4 & 5 & 6 & 7 & 8 \\
\midrule
R1
& $+2.67 \pm 4.25$
& $+4.33 \pm 5.42$
& $+2.33 \pm 2.58$
& $+1.67 \pm 2.58$
& $+1.00 \pm 0.00$
& $+0.00 \pm 0.00$
& $+0.33 \pm 0.97$ \\

AdaRFT
& $+2.67 \pm 0.97$
& $+5.00 \pm 6.08$
& $+3.33 \pm 3.51$
& $+2.00 \pm 0.00$
& $+1.33 \pm 0.97$
& $+0.00 \pm 0.00$
& $+0.33 \pm 0.97$ \\
\bottomrule
\end{tabular}
}
\end{table*}

\begin{table*}[t]
\centering
\small
\caption{Paired accuracy improvements (\%) of \textbf{SPaCe} over baselines on DeepScaleR-Easy and DeepScaleR-Difficult splits using Qwen3-0.6B-Instruct as the base LLM. Entries are mean $\pm$ 90\% confidence interval over three shared random seeds (Student's $t$-distribution, $df=2$).}

\label{tab:paired_deepscaler_splits}
\resizebox{\textwidth}{!}{%
\begin{tabular}{l|cccccccc}
\toprule
& \multicolumn{4}{c}{\textbf{DeepScaleR-Easy}} & \multicolumn{4}{c}{\textbf{DeepScaleR-Difficult}} \\
\cmidrule(lr){2-5}\cmidrule(lr){6-9}
\textbf{Baseline}
& \textbf{GSM8K} & \textbf{MATH500} & \textbf{AIME24} & \textbf{AIME25}
& \textbf{GSM8K} & \textbf{MATH500} & \textbf{AIME24} & \textbf{AIME25} \\
\midrule
R1
& $+1.40 \pm 1.36$ & $+2.13 \pm 0.00$ & $+6.67 \pm 9.77$ & $+5.56 \pm 2.92$
& $+0.60 \pm 0.96$ & $+0.40 \pm 1.63$ & $+5.56 \pm 14.14$ & $+6.67 \pm 4.89$ \\

AdaRFT
& $+1.47 \pm 0.91$ & $+3.40 \pm 0.78$ & $+4.45 \pm 5.78$ & $+2.22 \pm 5.78$
& $+2.23 \pm 1.15$ & $+2.07 \pm 1.46$ & $+2.23 \pm 9.22$ & $+7.78 \pm 15.56$ \\
\bottomrule
\end{tabular}%
}
\end{table*}

\begin{table*}[t]
\centering
\small
\caption{Paired accuracy improvements (\%) of \textbf{SPaCe} over \textbf{AdaRFT} using other base LLMs with DeepScaler-Uniform as the training data. Entries are mean $\pm$ 90\% confidence interval over three shared random seeds (Student's $t$-distribution, $df=2$).}
\label{tab:paired_other_LLMs_uniform_vs_adarft}
\resizebox{\textwidth}{!}{%
\begin{tabular}{l|c|cccc}
\toprule
\textbf{Baseline} & \textbf{Backbone}
& \textbf{GSM8K} & \textbf{MATH500} & \textbf{AIME24} & \textbf{AIME25} \\
\midrule
AdaRFT & Falcon3-1b-Instruct
& $+2.09 \pm 1.53$ & $+0.53 \pm 1.28$ & $+0.00 \pm 0.00$ & $+2.22 \pm 3.24$ \\

AdaRFT & Llama-3.2-1B-Instruct
& $+1.04 \pm 0.41$ & $+0.87 \pm 2.70$ & $+0.00 \pm 5.61$ & $+1.11 \pm 3.24$ \\

AdaRFT & Qwen2.5-0.5B-Instruct
& $+2.40 \pm 3.42$ & $+2.00 \pm 3.36$ & $+2.20 \pm 3.21$ & $+1.10 \pm 3.21$ \\

AdaRFT & Qwen3-8B-Base
& $+2.00 \pm 2.03$ & $+2.47 \pm 0.19$ & $+1.11 \pm 3.24$ & $+0.00 \pm 0.00$ \\

\bottomrule
\end{tabular}%
}
\end{table*}

\subsection{Ablation On Number Of Clusters} \label{sec:number_clusters}
In this section, we ablate on different number of clusters to see its effect on the final performance. Figure~\ref{fig:different_number_clusters} shows how SPaCe’s performance depends on the number of clusters \(K\). Accuracy peaks at a moderate value (\(K\)=7), while very few (\(K\)=1) or many (\(K\)=20) clusters reduce performance. This reflects a trade-off between specialization and generalization: too few clusters collapse diverse examples into coarse groups, while too many fragment the data, leading to sparse sampling and unstable training signals. A moderate clustering level provides the best balance, enabling the bandit to exploit informative variation without over-fragmentation. These results highlight the importance of tuning \(K\), as it directly affects how effectively the curriculum leverages per-example information.

\begin{figure}[t]  
    \centering
    \includegraphics[width=\columnwidth]{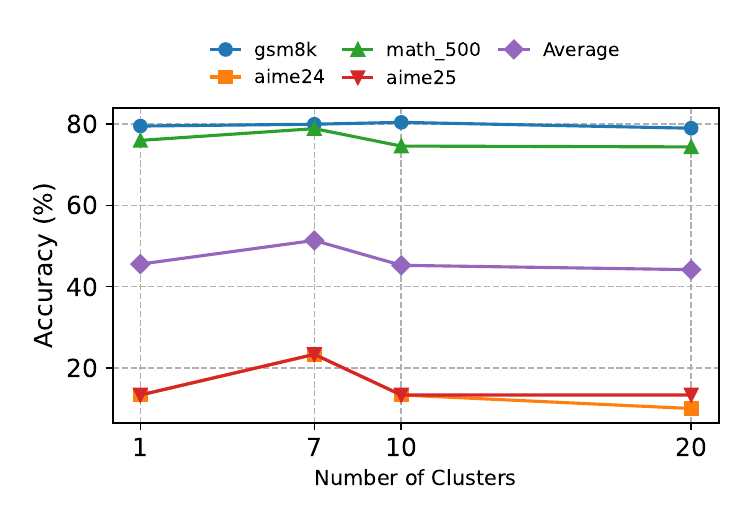}
    \caption{Results for 1 seed with Qwen3-0.6B and different number of clusters.}
    \label{fig:different_number_clusters}
\end{figure}

\subsection{Additional Baseline: Learning Impact Measurement} \label{sec:limr}
In this section, we compare the performance of our method on mathematical reasoning tasks against a variance-based data selection approach, Learning Impact Measurement (LIM) \citep{li2025limrrlscaling}. While the original LIM paper reports results on the MATH-Full dataset \citep{hendrycksmath2021}, we also conduct clustering on the same dataset using our method (SPaCe) to ensure a fair comparison. It is worth noting that after data reduction, LIM retains approximately 1,400 training samples, whereas SPaCe selects only 5 representative clusters, corresponding to just 50 data points, which amounts to \textbf{merely 3\% of the data used by LIM baseline}. Despite this drastic reduction, our method achieves competitive or superior results, demonstrating that SPaCe can reach high efficiency and effectiveness with a fraction of the training data. For these experiments, we perform training on 3 different seeds, and report in Table \ref{tab:limrbresults}.

\subsubsection{LIM definition}
LIM computes a per-sample score from its reward trajectory relative to the model’s average reward curve across epochs. Let \(r_i^{\,k}\) be the reward of sample \(i\) at epoch \(k\) and \(r_k = \tfrac{1}{N}\sum_{i=1}^{N} r_i^{\,k}\) the epoch-wise mean over all \(N\) samples for \(k=1,\dots,K\). The alignment score is
\begin{equation}
s_i \;=\; 1 \;-\;
\frac{\sum_{k=1}^{K}\bigl(r_i^{\,k}-r_k\bigr)^2}
     {\sum_{k=1}^{K}\bigl(1-r_k\bigr)^2},
\qquad i=1,\dots,N,    
\end{equation}

which normalizes the squared deviation of the sample trajectory from the epoch-wise mean. Data reduction is performed by thresholding:
\begin{equation}
\mathcal{D}_{\text{LIM}} \;=\; \{\, i \;:\; s_i > \theta \,\}.    
\end{equation}

\subsubsection{Results}
We show the results on four mathematical reasoning datasets between our method and LIM using MATH as training data in Table~\ref{tab:limrbresults}. We select Qwen2.5-0.5B-Instruct as the base LLM for training. 

The results highlight the effectiveness of SPaCe compared to both the Base model and LIM. On GSM8K, SPaCe reaches 32.0\%, which represents a relative improvement of +22\% over the Base (26.3\%) and still surpasses LIM (30.7\%). On MATH500, SPaCe achieves 20.7\%, outperforming LIM (20.3\%) and the Base (20.0\%), showing that our method yields more stable gains even on challenging competition-level problems. Notably, SPaCe is the only method that improves performance on the Olympiad benchmarks: it attains 1.1\% on AIME24 and 3.3\% on AIME25, while both the Base and LIM fail to make progress on these harder tasks.  

Overall, these findings confirm that SPaCe not only provides consistent improvements on standard benchmarks such as GSM8K and MATH500, but also uniquely enhances generalization to the most challenging settings, where variance-based selection methods like LIM struggle. This demonstrates the robustness and efficiency of our cluster-based approach in leveraging limited training data for stronger downstream reasoning performance.

\begin{table}[t]
\centering
\footnotesize
\setlength{\tabcolsep}{3.5pt}
\renewcommand{\arraystretch}{1.05}
\begin{tabular}{lcccc}
\toprule
Method & GSM8K & MATH500 & AIME24 & AIME25 \\
\midrule
Base & 26.3 & 20.0 & 0.0 & 0.0 \\
LIM & 30.7$_{1.3}$ & 20.3$_{1.1}$ & 0.0$_{0.0}$ & 0.0$_{0.0}$ \\
\rowcolor{gray!15}
SPaCe & \textbf{32.0$_{0.8}$} & \textbf{20.7$_{1.2}$} & \textbf{1.1$_{1.9}$} & \textbf{3.3$_{0.0}$} \\
\bottomrule
\end{tabular}
\caption{Mean$_{std}$ over 3 seeds on MATH benchmarks using Qwen2.5-0.5B-Instruct. Best in \textbf{bold}.}
\label{tab:limrbresults}
\end{table}

\subsection{Effect Of Diverse Sample Selection} \label{sec:diverse_sample_selection}
We show the impact of sample selection strategies for the data reduction on the performance of different LLMs in Figure \ref{fig:select_diverse_examples}. Specifically, we compare our method against two baselines: \textit{random}, which randomly selects training examples for the cluster, and \textit{closest}, which selects the closest examples to the cluster center. As observed, selecting diverse examples with our method consistently yields the highest performance across four datasets. Interestingly, the \textit{closest} strategy performs worse than \textit{random} in most cases. We hypothesize that this is because the examples nearest to the cluster center tend to be overly similar, thus failing to provide sufficient coverage and variation for effective learning.

\begin{figure}[t]
    \centering
    \includegraphics[width=\columnwidth]{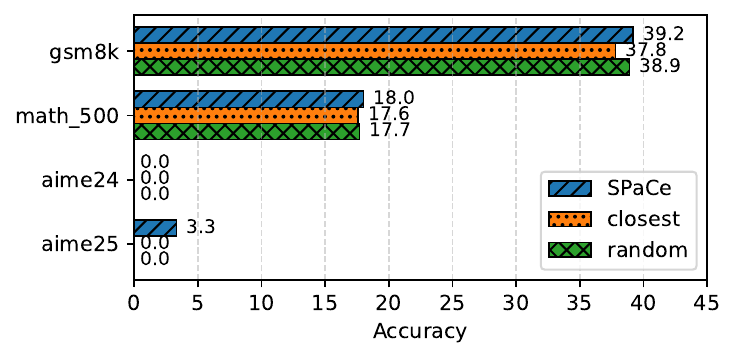}
    \caption{
    Comparison of selection strategies across datasets. Selecting diverse examples with SPaCe outperforms both random and closest baselines. Closest examples perform worse, likely due to reduced variety within each cluster.
    }
    \label{fig:select_diverse_examples}
\end{figure}

\subsection{Impact Of Embedding Model Choice} \label{sec:embedding_model_choice}
To assess SPaCe’s robustness to different semantic embedding backbones, we compare its default sentence encoder with an alternative based on \textit{Qwen2-1.5B-Instruct} (\textit{Alibaba-NLP/gte-Qwen2-1.5B-instruct}). We denote Qwen3\(^1\) as the baseline using \textit{Qwen3-Embedding-0.6B}, which is adopted in SPaCe, and Qwen2.5\(^2\) as the baseline using \textit{Alibaba-NLP/gte-Qwen2-1.5B-instruct}. Table~\ref{tab:different_sentence_embedding_model} reports zero-shot performance on four math reasoning benchmarks. Results show that SPaCe yields nearly identical performance across both embedding models, with the Qwen3-based encoder showing a slight advantage. This plug-and-play flexibility demonstrates that any high-quality pre-trained encoder can be seamlessly integrated into our framework without retraining.

\begin{table*}[h]
\centering

\begin{tabular}{l|c|c|c|c}
\toprule
\textbf{Embedding} 
& \textbf{GSM8K} 
& \textbf{MATH500} 
& \textbf{AIME24} 
& \textbf{AIME25} \\
\midrule

Qwen3$^{1}$ & 32.9$\pm$0.9 & 22.0$\pm$0.7 & 2.2$\pm$1.9 & 1.1$\pm$1.9 \\
Qwen2.5$^{2}$ & 31.6$\pm$0.6 & 21.0$\pm$1.2 & 1.1$\pm$1.9 & 0.0$\pm$0.0 \\
\bottomrule
\end{tabular}
\caption{Results using DeepScaleR as training data for Qwen2.5-0.5B-Instruct, evaluated with two embedding models: Qwen3$^{1}$ and Qwen2.5$^{2}$.}
\label{tab:different_sentence_embedding_model}
\end{table*}

\subsection{Impact Of The Number Of PCA Component} \label{sec:number_pca_component}
In SPaCe, we first apply PCA to reduce the dimensionality of the latent vectors extracted from the pretrained Sentence‑BERT model. By default we use 50 principal components; here, we vary this number to study its effect on final performance. Using Qwen3‑0.6B as the base model and training on the DeepScaleR‑uniform dataset, the results are shown in Table \ref{tab:pca_components}. We observe that smaller to moderate numbers of components (10 or 50) yield the best performance, whereas larger values (100 or 300) lead to a decline. We hypothesize that very high‑dimensional embeddings overwhelm the difficulty signal prior to clustering, resulting in poorer downstream performance.

\begin{table*}[h]
\centering
\begin{tabular}{l|c|c|c|c|c}
\toprule
\textbf{Number of PCA Components} 
& \textbf{GSM8K} 
& \textbf{MATH500} 
& \textbf{AIME24} 
& \textbf{AIME25} 
& \textbf{Average}\\
\midrule
10 & 79.2& 74.4& 13.3 & 20.0 & 46.7\\
50 & 80.0 & 78.9 & 23.3 & 13.3 & 48.9\\
100 & 79.8& 76.6 & 6.7& 13.3 & 44.1\\
300 & 79.0& 75.8& 16.7&10.0 & 45.4\\
\bottomrule
\end{tabular}
\caption{Results on 1 same seed on the DeepScaleR-uniform dataset using Qwen3-0.6B as the base model with different number of PCA components.}
\label{tab:pca_components}
\end{table*}

\subsection{Number Of Samples In Each Cluster} \label{sec:number_samples_each}
We vary the number of samples per cluster \(l\) to assess its impact, and present the results in Figure~\ref{fig:different_examples_per_cluster}. As shown, performance peaks at our default setting of \( l = 10 \). In contrast, using \( l = 1 \) yields lower performance, likely because the selected examples lack sufficient diversity to provide a strong learning signal. Larger values of \( l \in [100, 300] \) also lead to degraded performance, which we attribute to increased randomness when sampling too many examples per arm, especially when \( l \gg B \). Due to compute limits, we couldn not extensively tune $l$ between 10–100, where better results might be possible. Overall, setting \( l \sim B \) achieves the most reasonable results.

\begin{figure}[h]
  \centering
  \includegraphics[width=\columnwidth]{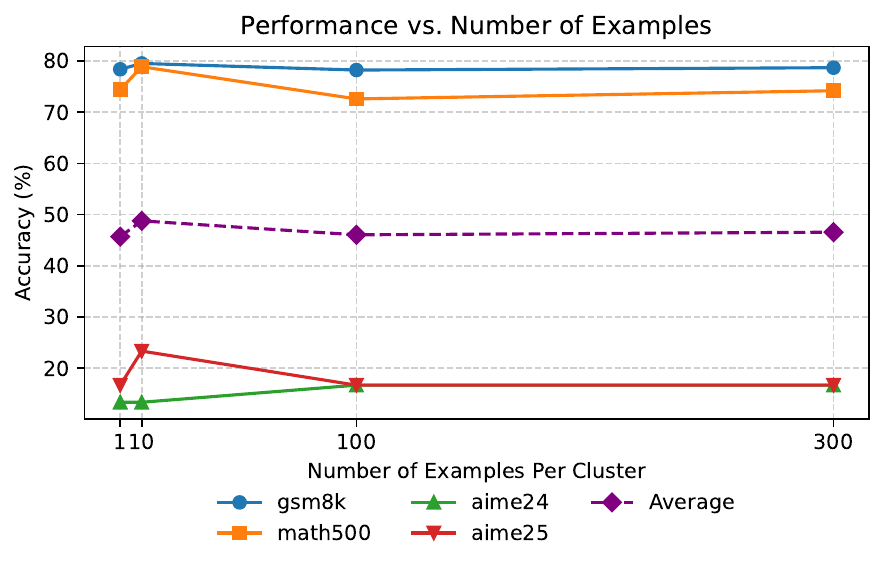}
  \caption{Results for 1 seed with Qwen3-0.6B and different number of samples per cluster.}
  \label{fig:different_examples_per_cluster}
\end{figure}

\subsection{Effect of Removing Difficulty in Clustering on Performance} \label{sec:remove_difficulty}
To isolate the contribution of the per-example difficulty attribute to cluster quality, the clustering pipeline is repeated after ablating this signal (i.e., omitting the difficulty attribute). Results are reported in Figure~\ref{fig:remove_difficulty}, which shows downstream performance of \textit{Qwen3-0.6B} on each dataset under this “no-difficulty” condition. Across benchmarks, performances consistently drop when the difficulty attribute is excluded, indicating that this signal is important for forming meaningful clusters and improving curriculum selection.

\begin{figure}[t]
    \centering
    \includegraphics[width=\columnwidth]{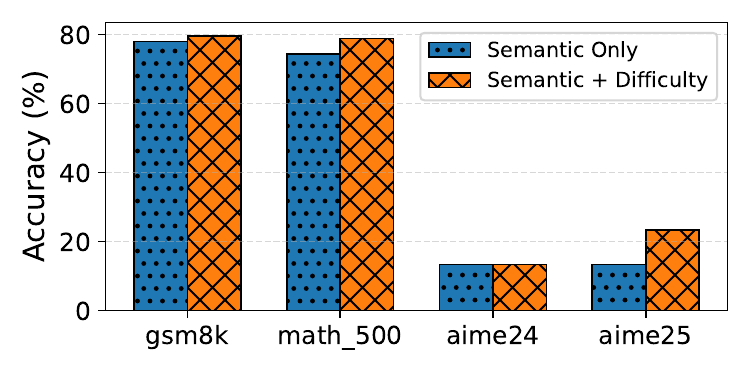}
    \caption{Effect of Removing Difficulty in Clustering on Performance.}
    \label{fig:remove_difficulty}
\end{figure}

\subsection{Additional Related Work: Reinforcement Fine-Tuning for Language Models}
Due to page limit in the main paper, we include the additional related work in the Reinforcement Fine-Tuning methods for Language Models here.

Language Models can be formulated as sequential decision-making agents, enabling the application of RL techniques for fine-tuning. Proximal Policy Optimization (PPO) \citep{schulman2017proximalpolicyoptimizationalgorithms} has been widely adopted in early RLHF pipelines due to its balance between stability and sample efficiency. More recent work introduced actor-only alternatives such as REINFORCE++ \citep{hu2025reinforce++} and Group Relative Policy Optimization (GRPO) \citep{deepseekai2025deepseekr1incentivizingreasoningcapability}, which eliminate the need for value networks and have shown strong performance on large language models, particularly in reasoning tasks. By avoiding a separately trained critic, these approaches simplify optimization, reduce variance in policy updates, and mitigate instability caused by poorly estimated value functions.
GRPO, in particular, has been successfully deployed in large-scale instruction tuning setups where explicit reward modeling is either impractical or misaligned with target behaviors. Instead of depending on hand-crafted reward models, GRPO leverages group-based relative comparisons across sampled trajectories, thereby aligning the optimization signal with preference-style supervision. This actor-only paradigm aligns naturally with recent trends in LLM alignment, where scalability, reduced computational overhead, and robustness to noisy feedback are critical. These methods represent a shift from critic-dependent RLHF pipelines toward lightweight, actor-centric algorithms that better match the scale and complexity of modern LLM training regimes.

\subsection{Selected Training Samples Difficulty} \label{sec:selected_training_samples}
We analyze the selected training examples to understand how SPaCe constructs its curriculum. Figure~\ref{fig:difficulty_distribution} shows the difficulty distribution of questions selected on the DeepScaleR-uniform dataset. SPaCe consistently chooses examples across the full difficulty range—from easy (near 0) to hard (near 100)—ensuring balanced coverage. This diversity enables training on a broad range of problems, avoiding overfitting to simple or complex cases. Notably, this balance emerges without manual difficulty constraints, highlighting the effectiveness of SPaCe’s clustering and selection strategy.

\begin{figure*}[h]
    \centering
    \includegraphics[width=\linewidth]{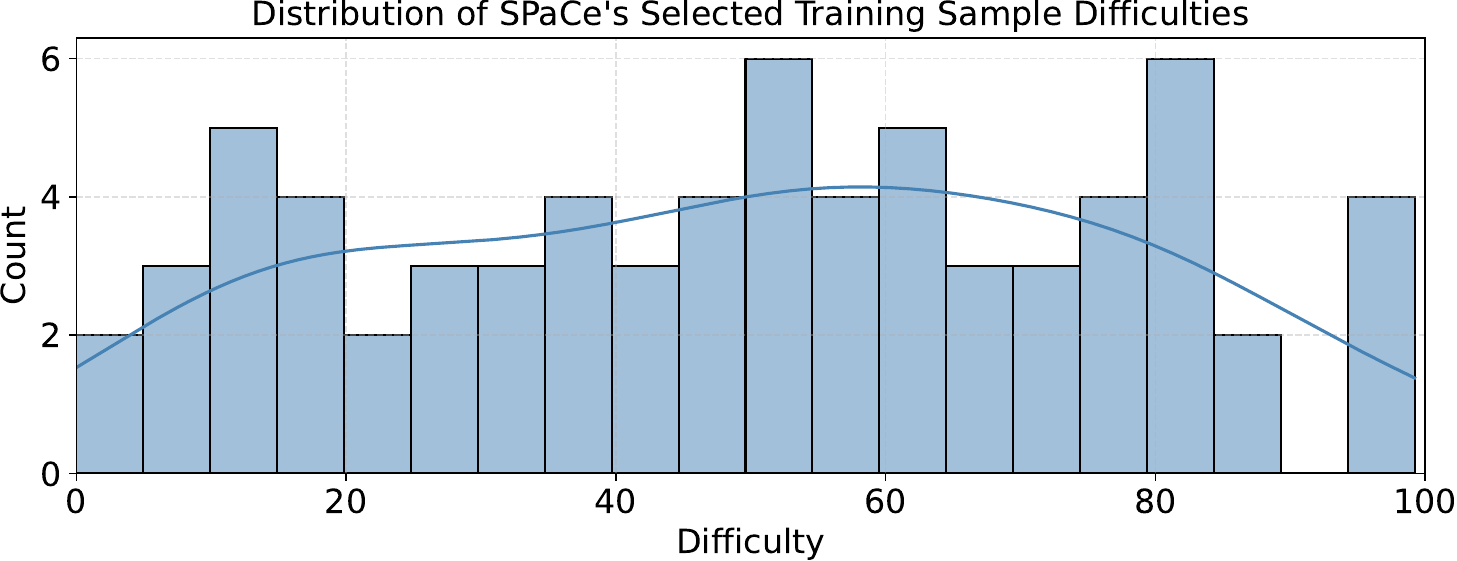}
    \caption{
    Difficulty distribution of training examples in SPaCe.
    }
    \label{fig:difficulty_distribution}
\end{figure*}

\subsection{Deepscaler Subset Distributions} \label{sec:datset_dist}
We provide the difficulty score distributions of three DeepScaleR subsets: DeepScaleR Uniform, DeepScaleR Easy, and DeepScaleR Difficult, as shown in Figure~\ref{fig:difficulty_distribution_subset}. Each subset exhibits distinct difficulty characteristics, reflecting the varying levels of challenge present in the data. The distributions are grouped into bins of size 10, allowing for a clear comparison of how problem difficulty varies across these subsets. In particular, the Uniform subset spans the entire difficulty range with roughly balanced coverage, making it suitable for general-purpose training and evaluation. By contrast, the Easy subset is concentrated heavily in the lower-difficulty bins, highlighting its role in providing simpler problems for warm-up training or curriculum learning. Meanwhile, the Difficult subset skews strongly toward the higher-difficulty bins, offering more challenging samples that are valuable for stress-testing reasoning capabilities and benchmarking advanced methods. Together, these subsets provide complementary perspectives on model performance across a wide spectrum of difficulty levels, ensuring a more comprehensive assessment of reasoning ability.

\begin{figure*}[h]
    \centering
    \includegraphics[width=\linewidth]{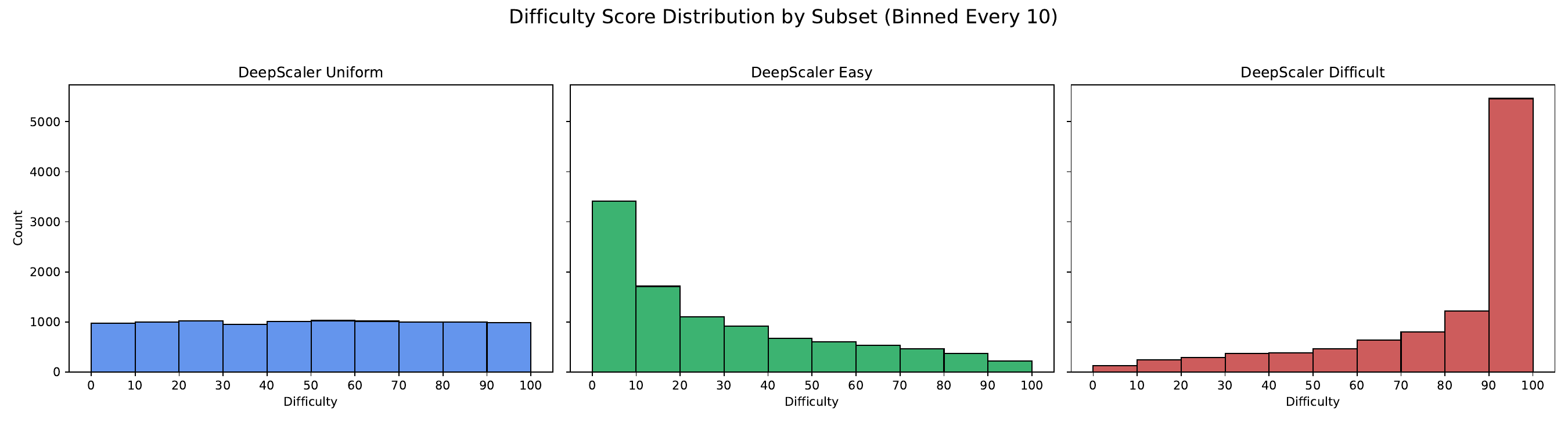}
    \caption{DeepScaleR subsets' difficulty distributions.}
    \label{fig:difficulty_distribution_subset}
\end{figure*}

\subsection{Cluter Analysis} \label{sec:cluster_analysis}
In this section, to provide further insight into what happens during the clustering phase of our framework, we analyze several representative settings to examine both what is captured for training and how clustering shapes the overall behavior of our approach. Specifically, we investigate how problem embeddings are grouped and how these clusters align with meaningful curricular categories. For this purpose, we employ GPT-5’s API \citep{openai2025} to categorize each problem into one of the seven canonical subject areas defined by \citet{hendrycksmath2021}, namely: Prealgebra, Algebra, Number Theory, Counting \& Probability, Geometry, Intermediate Algebra, and Precalculus. This taxonomy is consistent with the Art of Problem Solving (AoPS) curriculum \citep{aops_wiki}, which provides a widely accepted structure for organizing mathematical problem-solving skills.  

\begin{itemize}
  \item \textbf{Prealgebra}: Covers arithmetic foundations, including fractions, decimals, percents, ratios, proportions, and basic number properties. It also introduces simple equations and word problems.

  \item \textbf{Algebra}: Focuses on symbolic manipulation and equations, such as linear and quadratic equations, inequalities, systems of equations, factoring, functions, and exponents. It marks the transition from arithmetic to general algebraic reasoning.

  \item \textbf{Number Theory}: Includes topics such as divisibility, prime numbers, greatest common divisors, modular arithmetic, congruences, and Diophantine equations. Problems emphasize reasoning about integer structure and properties.

  \item \textbf{Counting \& Probability}: Encompasses combinatorics and elementary probability, including permutations, combinations, casework, binomial coefficients, expected value, and probabilistic reasoning.

  \item \textbf{Geometry}: Centers on Euclidean geometry of lines, angles, triangles, quadrilaterals, circles, and polygons. Topics include similarity, congruence, area, volume, coordinate geometry, and introductory trigonometric methods.

  \item \textbf{Intermediate Algebra}: Extends algebra with higher-level topics such as polynomials, rational functions, complex numbers, inequalities, logarithmic and exponential functions, and sequences/series.

  \item \textbf{Precalculus}: Prepares for calculus through trigonometry, advanced functions, polar/parametric representations, vectors, and deeper study of sequences and series.
\end{itemize}

\subsubsection{DeepScaleR Dataset}
We analyze the \textit{DeepScaleR} dataset under the configuration with $K{=}7$ clusters induced by \texttt{Qwen3-0.6B-Embeddings}, consistent with Table~\ref{tab:number_of_clusters}. Figure~\ref{fig:difficulty_distribution_subset} visualizes one representative run. The resulting partition exhibits intuitive curricular structure: a majority of items fall into \textit{Prealgebra} (51.4\%), with \textit{Geometry} (18.6\%) and \textit{Algebra} (12.9\%) also comprising substantial shares; by contrast, smaller categories such as \textit{Number Theory} (1.4\%) and \textit{Intermediate Algebra} (1.4\%) are comparatively scarce, with the remaining mass distributed across \textit{Counting \& Probability} and \textit{Precalculus}. 

Beyond mirroring topical prevalence in the underlying corpus, this distribution suggests that the embedding-driven clustering is aligned with both latent difficulty and high-level curricular distinctions. Practically, this yields two benefits for downstream training: (i) it avoids over-emphasizing any single subject area, providing balanced exposure across topics; and (ii) it simplifies scheduling, since strata can be sampled in a principled way (e.g., uniformly or by difficulty-aware policies) without ad-hoc reweighting to correct for cluster idiosyncrasies. In short, even though clusters are formed in representation space, they preserve pedagogically meaningful boundaries that support stable and fair curriculum design.

\begin{figure*}[t]
    \centering
    \includegraphics[width=\linewidth]{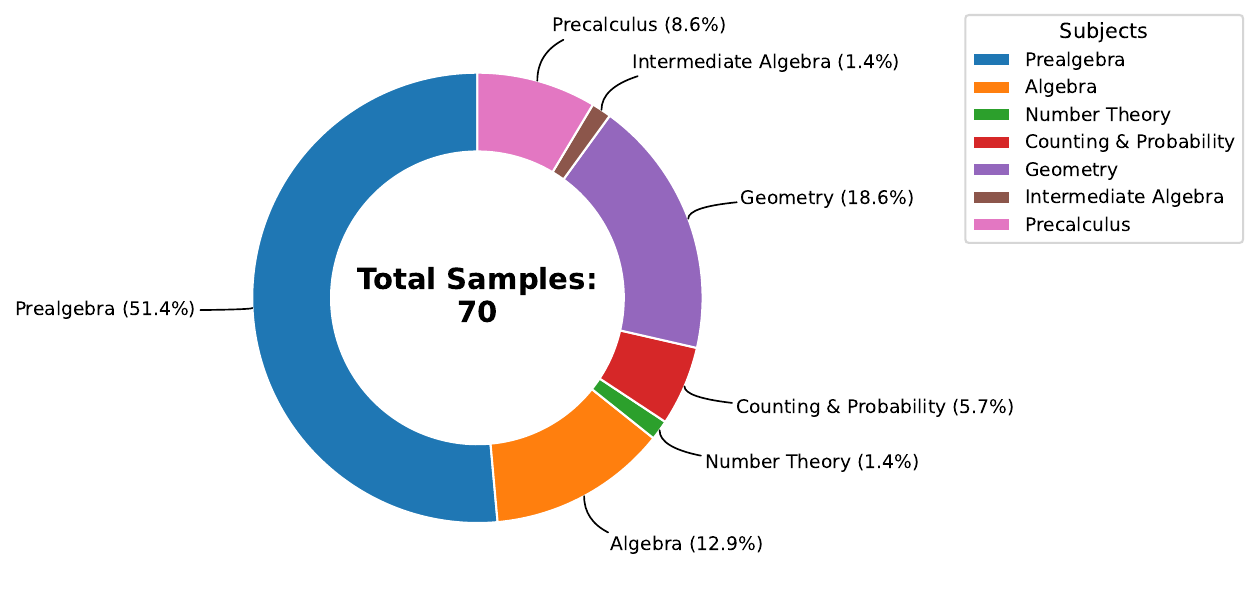}
    \caption{Selected data subjects (70 samples) using Qwen3-0.6B-Embedding. The data is clustered from DeepScaleR-uniform set.}
    \label{fig:math_diff}
\end{figure*}

\begin{figure*}[t]
    \centering
    \includegraphics[width=\linewidth]{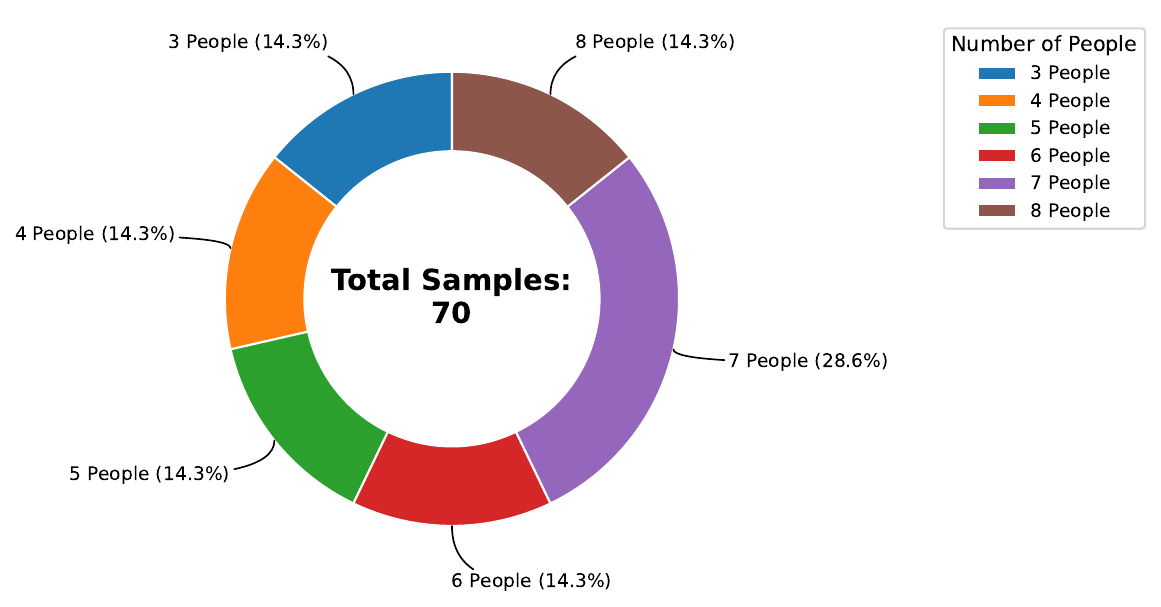}
    \caption{Selected data subjects (70 samples) using Qwen3-0.6B-Embedding. The data is clustered from Knights and Knaves set.}
    \label{fig:kk_diff}
\end{figure*}

\subsubsection{Knights-and-Knaves Dataset}
We conduct an analogous analysis on the \textit{Knights-and-Knaves} dataset using $K{=}7$ clusters obtained with \texttt{Qwen3-0.6B-Embeddings}. The results are shown in Figure~\ref{fig:kk_diff}. Plotting the distribution of selected questions by the number of people per instance reveals two robust regularities across all three training seeds: (i) no 2-person questions are selected; and (ii) aside from the 7-person category, all remaining categories contain the same number of questions. The same symmetry appears in the other seeds, indicating that the clustering process is not only semantically coherent but also structurally consistent with a salient, coarse-grained attribute (the number of entities in the prompt).

These regularities have useful practical consequences. First, they provide balanced coverage over interaction sizes, preventing the curriculum from drifting toward a single conversational complexity. Second, they reduce confounds in subsequent evaluation and scheduling: because most categories are equalized, one can adopt simple, uniform sampling or layer a performance-aware scheduler on top without introducing artifacts from cluster imbalance. We reckon the persistent absence of 2-person items likely reflects a combination of dataset composition and our selection protocol’s preference for more discriminative examples.

\subsection{Training Time} \label{sec:trainingtime} 
To evaluate computational efficiency, we measure wall-clock training time when applying AdaRFT and our proposed SPaCe across five representative base models: \textit{Qwen3-0.6B}, \textit{Falcon}, \textit{Llama3}, \textit{Qwen2.5}, and \textit{Qwen3-8B}. These models span a range of parameter scales and architectures, providing a balanced testbed for assessing runtime behavior under different backbone choices. All runs are conducted under identical hardware and data conditions to ensure a fair comparison. We also note that both methods require precomputed hardness scores; we therefore follow the same hardness-estimation protocol as AdaRFT \citep{AdaRFT} for all methods and report only the training time here.

Figure~\ref{fig:trainingtime} reports the results. Across all models, SPaCe consistently reduces wall-clock training time relative to AdaRFT, with per-model savings ranging from $2.1\%$ to $10.8\%$ (e.g., Qwen2.5: $-66$ minutes, $-10.8\%$; Qwen3\texttt{-}8B: $-85$ minutes, $-5.6\%$; Llama3: $-40$ minutes, $-7.4\%$; Qwen3\texttt{-}0.6B: $-23$ minutes, $-2.4\%$; Falcon: $-14$ minutes, $-2.1\%$).

We attribute these savings to SPaCe’s clustering phase, which dynamically prioritizes clusters that yield higher learning signal over a reduced selection space, thereby avoiding wasted updates on redundant or low-yield samples. Although the absolute magnitude of savings depends on the underlying backbone, the improvements are consistent across diverse architectures, highlighting that SPaCe not only improves data efficiency but also offers a practical reduction in training cost without additional engineering or inference-time overhead.

\begin{figure*}[h]
    \centering
    \includegraphics[width=\textwidth]{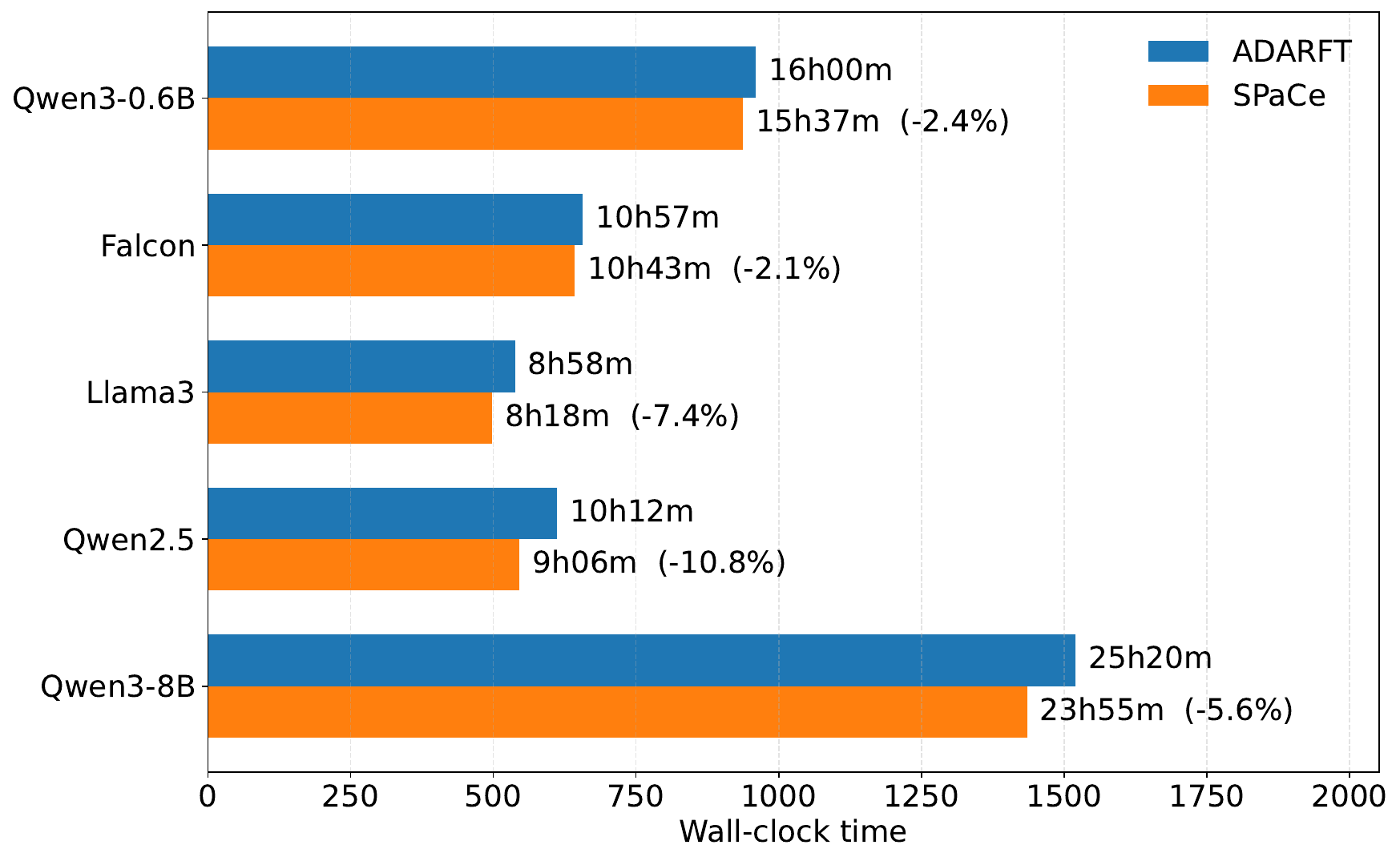}
    \caption{Training time comparison between AdaRFT and SPaCe.}
    \label{fig:trainingtime}
\end{figure*}

\subsection{SPaCe training details} \label{appendix:SPaCe_details}

\subsubsection{Training Hyperparameters}
\paragraph{General Training Parameters}
In this section, we provide the training details of SPaCe in Table \ref{tab:SPaCe_parameters}. These parameters apply for full training LLMs (which excludes the training of \textit{Qwen3-8B-Base}).

\begin{table*}[h]
\centering
\begin{tabular}{l|c}
\toprule
\textbf{Parameters} & \textbf{Value} \\
\midrule
 Number of consecutive steps for penalizing (\(T_{consecutive}\)) & 10 \\
 Number of examples per cluster (\(l\)) & 10\\
 Number of PCA components & 50\\
 Batch size (B) & 8\\
 Number of generation per step (G) & 8\\
 Maximum completion length (L) & 1200\\
 Initial learning rate (\(\alpha\)) & $5e^{-6}$\\
 Weight Decay& 0.1\\
 Warmup Ratio & 0.1 \\
 lr\_scheduler\_type & cosine \\
 Adam \(\beta_1\) & 0.9\\
 Adam \(\beta_2\) & 0.99 \\
 bf16 & True\\
 Per device train batch size & 8\\
 Gradient accumulation steps & 8\\
 Max grad norm (\(G_{norm}\)) & 0.1 \\
 $\epsilon$ & 1e\(^{-6}\)\\
\bottomrule
\end{tabular}
\caption{Parameters used in SPaCe.}
\label{tab:SPaCe_parameters}
\end{table*}

\paragraph{LoRA Training Parameters}
In this section, we provide the LoRA training parameters for \textit{Qwen3-8B-Base}. All the parameters used are reported in Table \ref{tab:lora_parameters}.

\begin{table*}[h]
\centering
\begin{tabular}{l|c}
\toprule
\textbf{Parameters} & \textbf{Value} \\
\midrule
 All Parameters & 8,194,569,216\\
 Trainable Parameters (\(l\)) & 3,833,856\\
 Trainable \% & 0.05\\
 Rank & 8 \\
 LoRA $\alpha$ & 16 \\
 Target Modules & \text{["q\_proj", "v\_proj"]}\\
 LoRA Dropout & 0.1\\
 Bias & None \\
\bottomrule
\end{tabular}
\caption{Parameters used for Low-rank Adaptation (LoRA) Fine-tuning.}
\label{tab:lora_parameters}
\end{table*}

\subsubsection{Number of clusters}
We provide the number of clusters used for different settings of SPaCe in Table~\ref{tab:number_of_clusters}. While the optimal number of clusters varies across datasets, it remains within a moderate range, consistent with our observations in Section~\ref{sec:clustering_effects}.

\begin{figure}[h]
    \centering
    \includegraphics[width=\columnwidth]{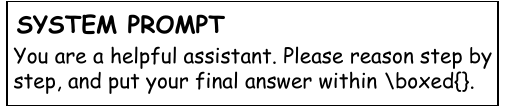}
    \caption{System prompt used in our experiments.}
    \label{fig:system_prompt}
\end{figure}

\begin{table*}[h]
\centering
\begin{tabular}{l|c|c}
\toprule
\textbf{Model} 
& \textbf{Train dataset} 
& \textbf{Number of clusters} 
\\
\midrule
\multirow{4}{*}{Qwen3-0.6B} 
  & DeepScaleR-uniform   &  7\\ 
  & DeepScaleR-easy      &  8\\
  & DeepScaleR-difficult &  10\\
  & GSM8K                &  10\\
\midrule
\multirow{1}{*}{DeepSeek-R1-Distill-Qwen-1.5B} 
  & DeepScaleR-uniform   &  7\\ 
\midrule
\multirow{1}{*}{Falcon3-1B-Instruct} 
  & DeepScaleR-uniform   &  6\\ 
\midrule
\multirow{1}{*}{Llama-3.2-1B-Instruct} 
  & DeepScaleR-uniform   &  8\\ 
\midrule
\multirow{1}{*}{Qwen2.5-0.5B-Instruct} 
  & DeepScaleR-uniform   &  7\\ 
\midrule
\multirow{1}{*}{Qwen3-8B-Base} 
  & DeepScaleR-uniform   &  7\\ 
\bottomrule
\end{tabular}
\caption{Number of clusters used for different settings of our method.}
\label{tab:number_of_clusters}
\end{table*}

\subsubsection{Model and data references}
We list the links to the LLM models and datasets in Table \ref{tab:item_urls}. 

\begin{table*}[t]
\centering

\begin{tabular}{l|l}
\hline
\textbf{Models/Datasets} & \textbf{URL} \\
\hline
Qwen3-Embedding-0.6B & \url{https://huggingface.co/Qwen/Qwen3-Embedding-0.6B} \\
Qwen2.5-0.5B-Instruct & \url{https://huggingface.co/Qwen/Qwen2.5-0.5B-Instruct}\\
Llama3.2-1B-Instruct & \url{https://huggingface.co/meta-llama/Llama-3.2-1B-Instruct}\\
Falcon3-1B-Instruct & \url{https://huggingface.co/tiiuae/Falcon3-1B-Instruct} \\
Alibaba-NLP/& \\
gte-Qwen2-1.5B-instruct & \url{https://huggingface.co/Alibaba-NLP/gte-Qwen2-1.5B-instruct} \\
DeepScaleR & \url{https://huggingface.co/datasets/agentica-org/DeepScaleR-Preview-Dataset} \\
GSM8K & \url{https://huggingface.co/datasets/openai/gsm8k} \\
MATH-500 & \url{https://huggingface.co/datasets/HuggingFaceH4/MATH-500} \\
AIME24& \url{https://huggingface.co/datasets/math-ai/aime24} \\
AIME25& \url{https://huggingface.co/datasets/math-ai/aime25} \\
\hline
\end{tabular}
\caption{Models and Datasets Details.}
\label{tab:item_urls}
\end{table*}

\subsubsection{System Prompt}
Following \citep{openr1}, the system prompt asks the model to generate the answer with clear requirements, with reasoning and answer following the format, as described in Figure \ref{fig:system_prompt}.

\subsection{Response Examples}
We present several sampled responses of SPaCe in Table \ref{tab:SPaCe_ex1} and \ref{tab:SPaCe_ex2}.

\begin{table*}[t]
\centering
\caption{Qwen3-0.6B fine-tuned with SPaCe responses.}
\begin{tabular}{|p{0.95\textwidth}|}
\hline
\textbf{Question} \\
Let $f$ be the function defined by $f(x)=ax^2-\sqrt{2}$ for some positive $a$. If $f(f(\sqrt{2}))=-\sqrt{2}$ then $a=$. Please reason step by step, and put your final answer within \texttt{\textbackslash boxed\{\}}. \\
\textbf{Answer} \\
To solve for $a$, we need to find the value of $a$ such that $f(f(\sqrt{2})) = -\sqrt{2}$, where $f(x) = ax^2 - \sqrt{2}$.                                                      
First, let's compute $f(\sqrt{2})$:                                                                                                                   

$$                                                                                                                                                    
f(\sqrt{2}) = a(\sqrt{2})^2 - \sqrt{2} = a(2) - \sqrt{2} = 2a - \sqrt{2}                                                                              
$$                                                                                                                                                                                                                                                                                                                                     
Now, we need to compute $f(f(\sqrt{2}))$:                                                                                                                                        
$$                                                                                                                                                                  
f(f(\sqrt{2})) = f(2a - \sqrt{2}) = a(2a - \sqrt{2})^2 - \sqrt{2}                                                                                                   
$$                                                                                                                                                                       
Let's expand the square term:

$$           
(2a - \sqrt{2})^2 = (2a)^2 - 2(2a)(\sqrt{2}) + (\sqrt{2})^2 = 4a^2 - 4a\sqrt{2} + 2
$$                                                                                                                                                                                                                                                                                                                                        
\[\dots\]

So the discriminant is zero, which means there is a repeated root:                                                                

$$                                                                                                                                
a = \frac{2\sqrt{2}}{4} = \frac{\sqrt{2}}{2}                                                                                      
$$                                                                                                                                

Therefore, the value of $a$ is $\frac{\sqrt{2}}{2}$.                                                                              

Let me double-check the calculations to ensure there are no errors. Starting from the beginning:                                  

1. $f(\sqrt{2}) = a(2) - \sqrt{2} = 2a - \sqrt{2}$                                                                                

2. Then $f(f(\sqrt{2})) = f(2a - \sqrt{2}) = a(2a - \sqrt{2})^2 - \sqrt{2}$                                                       

3. Expanding $(2a - \sqrt{2})^2$ gives $4a^2 - 4a\sqrt{2} + 2$                                                                    

4. Multiplying by $a$ and subtracting $\sqrt{2}$ gives $4a^3 - 4a^2\sqrt{2} + 2a - \sqrt{2}$                                      

5. Setting equal to $-\sqrt{2}$ leads to the equation $4a^3 - 4a^2\sqrt{2} + 2a = 0$                                              

6. Factoring gives $2a(2a^2 - 2a\sqrt{2} + 1) = 0$, leading to the quadratic with discriminant zero.                              

Yes, the steps seem correct. The discriminant being zero implies a repeated root, which gives $a = \frac{\sqrt{2}}{2}$. Therefore, the answer is $\boxed{\frac{\sqrt{2}}{2}}$.                

**Final Answer**                                                                                                                  
$$                                                                                                                                
\boxed{\frac{\sqrt{2}}{2}}                                                                                                        
$$                                                                                                                                
\\
\hline
\end{tabular}
\label{tab:SPaCe_ex1}
\end{table*}

\begin{table*}[t]
\centering
\caption{Qwen3-0.6B fine-tuned with SPaCe responses.}
\begin{tabular}{|p{0.95\textwidth}|}
\hline
\textbf{Question} \\

A set $S$ of points in the $xy$-plane is symmetric about the origin, both coordinate axes, and the line $y=x$. If $(2,3)$ is in $S$, what is the smallest number of points in $S$?. Please reason step by step, and put your final answer within \texttt{\textbackslash boxed\{\}}.\\

\textbf{Answer} \\
To solve this problem, we need to determine the smallest number of points in the set $ S $, given that the set is symmetric about the origin, both coordinate axes, and the line $ y = x $. Additionally, the point $ (2, 3) $ is in $ S $.    
 
---                                

\#\#\# Step 1: Understand the Symmetry Conditions                                                                                                                                                                                                  

A set of points that is symmetric about the **origin** means that for every point $ (x, y) $ in $ S $, $ (-x, -y) $ is also in $ S $.

Symmetry about the **coordinate axes** means that if $ (x, y) $ is in $ S $, then $ (y, x) $ is also in $ S $.

Symmetry about the **line $ y = x $** means that if $ (x, y) $ is in $ S $, then $ (y, x) $ is also in $ S $. 

\[\dots\]

Is there a point that we might be missing?

Let's check if $ (3, 2) $ is symmetric to itself. Yes, it is.

So, with these three points, we have covered all the symmetries.

But wait! What about the point $ (2, 3) $? Is there a point symmetric to it that we haven't considered?

No, because we already considered all the points under the given symmetries.

So, the minimal number of points in $ S $ is 3.

---

\#\#\# Final Answer:

$$\boxed{3}$$\\
\hline

\end{tabular}
\label{tab:SPaCe_ex2}
\end{table*}

\subsection{Selected Example Analysis}
In this section, we provide details on the selected examples by SPaCe. For each cluster, we show 2 selected examples selected by our method using \textit{Qwen3-0.6B}, with the DeepScaleR-uniform dataset. We select the examples from the setting with total number of clusters equal 7. The examples are shown from Table \ref{tab:ex1} to Table \ref{tab:last_ex}.

\begin{table*}[t]
\centering
\caption{Cluster 0 examples. Only 2 examples are shown for illustration.}
\begin{tabular}{|p{0.95\textwidth}|}
\hline
\textbf{Example 1} \\
\textbf{Prompt} \\
(In the preliminaries of optimal method and experimental design) When using the 0.618 method to find the optimal amount to add in an experiment, if the current range of excellence is $[628, 774]$ and the good point is 718, then the value of the addition point for the current experiment is \_\_\_\_\_\_. Please reason step by step, and put your final answer within \texttt{\textbackslash boxed\{\}}. \\
\textbf{Answer} \\
684 \\
\hline
\textbf{Example 2} \\
\textbf{Prompt} \\
Calculate the probability that in a deck of 52 cards, the second card has a different suit than the first, and the third card has a different suit than the first and second. Please reason step by step, and put your final answer within \texttt{\textbackslash boxed\{\}}. \\
\textbf{Answer} \\
$\frac{169}{425}$ \\
\hline
\end{tabular}
\label{tab:ex1}

\end{table*}

\begin{table*}[t]
\centering
\caption{Cluster 1 examples. Only 2 examples are shown for illustration.}
\begin{tabular}{|p{0.95\textwidth}|}
\hline
\textbf{Example 1} \\
\textbf{Prompt} \\
Calculate: $\frac{{\cos190^{\circ}(1+\sqrt{3}\tan10^{\circ})}}{{\sin290^{\circ}\sqrt{1-\cos40^{\circ}}}} = \_\_\_\_\_$. Please reason step by step, and put your final answer within \texttt{\textbackslash boxed\{\}}. \\
\textbf{Answer} \\
$2\sqrt{2}$ \\
\hline
\textbf{Example 2} \\
\textbf{Prompt} \\
Let $a_n$ be the number of $n$-digit numbers formed using only digits 1 and 2 such that no two adjacent digits are both 2. Find $a_5$. Please reason step by step, and put your final answer within \texttt{\textbackslash boxed\{\}}. \\
\textbf{Answer} \\
$13$ \\
\hline
\end{tabular}
\end{table*}

\begin{table*}[t]
\centering
\caption{Cluster 2 examples. Only 2 examples are shown for illustration.}
\begin{tabular}{|p{0.95\textwidth}|}
\hline
\textbf{Example 1} \\
\textbf{Prompt} \\
You have 5 red balls and 5 blue balls in a box. You randomly draw 4 balls without replacement. What is the probability that exactly 2 red balls are drawn? Please reason step by step, and put your final answer within \texttt{\textbackslash boxed\{\}}. \\
\textbf{Answer} \\
$\frac{25}{63}$ \\
\hline
\textbf{Example 2} \\
\textbf{Prompt} \\
If a and b are real numbers such that $a^2 + b^2 = 1$, what is the maximum value of $ab$? Please reason step by step, and put your final answer within \texttt{\textbackslash boxed\{\}}. \\
\textbf{Answer} \\
$\frac{1}{2}$ \\
\hline
\end{tabular}
\end{table*}

\begin{table*}[t]
\centering
\caption{Cluster 3 examples. Only 2 examples are shown for illustration.}
\begin{tabular}{|p{0.95\textwidth}|}
\hline
\textbf{Example 1} \\
\textbf{Prompt} \\
Solve for $x$: $\log_3(x^2 - 1) = 2$. Please reason step by step, and put your final answer within \texttt{\textbackslash boxed\{\}}. \\
\textbf{Answer} \\
$4$ \\
\hline
\textbf{Example 2} \\
\textbf{Prompt} \\
Evaluate the integral $\int_0^1 x e^x \, dx$. Please reason step by step, and put your final answer within \texttt{\textbackslash boxed\{\}}. \\
\textbf{Answer} \\
$e - 2$ \\
\hline
\end{tabular}
\end{table*}

\begin{table*}[t]
\centering
\caption{Cluster 4 examples. Only 2 examples are shown for illustration.}
\begin{tabular}{|p{0.95\textwidth}|}
\hline
\textbf{Example 1} \\
\textbf{Prompt} \\
If $\sin x + \cos x = \sqrt{2}$, find the value of $\sin^4 x + \cos^4 x$. Please reason step by step, and put your final answer within \texttt{\textbackslash boxed\{\}}. \\
\textbf{Answer} \\
$\frac{3}{4}$ \\
\hline
\textbf{Example 2} \\
\textbf{Prompt} \\
Find the sum of the series $\sum_{n=1}^{\infty} \frac{1}{n(n+1)}$. Please reason step by step, and put your final answer within \texttt{\textbackslash boxed\{\}}. \\
\textbf{Answer} \\
$1$ \\
\hline
\end{tabular}
\end{table*}

\begin{table*}[t]
\centering
\caption{Cluster 5 examples. Only 2 examples are shown for illustration.}
\begin{tabular}{|p{0.95\textwidth}|}
\hline
\textbf{Example 1} \\
\textbf{Prompt} \\
How many 4-digit numbers are there such that no two adjacent digits are the same? Please reason step by step, and put your final answer within \texttt{\textbackslash boxed\{\}}. \\
\textbf{Answer} \\
$5832$ \\
\hline
\textbf{Example 2} \\
\textbf{Prompt} \\
If $A = \{1,2,3,4\}$ and $B = \{3,4,5,6\}$, what is $A \cup B$? Please reason step by step, and put your final answer within \texttt{\textbackslash boxed\{\}}. \\
\textbf{Answer} \\
$\{1,2,3,4,5,6\}$ \\
\hline
\end{tabular}
\end{table*}

\begin{table*}[t]
\centering
\caption{Cluster 6 examples. Only 2 examples are shown for illustration.}
\begin{tabular}{|p{0.95\textwidth}|}
\hline
\textbf{Example 1} \\
\textbf{Prompt} \\
What is the value of the determinant of the matrix $\begin{bmatrix}1 & 2 \\ 3 & 4\end{bmatrix}$? Please reason step by step, and put your final answer within \texttt{\textbackslash boxed\{\}}. \\
\textbf{Answer} \\
$-2$ \\
\hline
\textbf{Example 2} \\
\textbf{Prompt} \\
Simplify: $(2x - 3)^2 - (x + 1)^2$. Please reason step by step, and put your final answer within \texttt{\textbackslash boxed\{\}}. \\
\textbf{Answer} \\
$3x^2 - 14x + 8$ \\
\hline
\end{tabular}
\label{tab:last_ex}

\end{table*}

\end{document}